\documentclass{article} 

\usepackage{iclr2027_conference,times}
\iclrfinalcopy

\usepackage[T1]{fontenc}
\usepackage[utf8]{inputenc}
\usepackage{microtype}          

\usepackage{amsmath}
\usepackage{amssymb}            
\usepackage{amsthm}             
\usepackage{bm}                 
\usepackage{nicefrac}           
\usepackage{xparse}             

\usepackage{graphicx}
\usepackage{float}              
\usepackage{booktabs}           
\usepackage{caption}
\usepackage{subcaption}         
\usepackage{wrapfig}            
\usepackage{needspace}          
\usepackage{xcolor}

\usepackage{algorithm}          
\usepackage{algorithmic}        

\usepackage{url}
\usepackage[hidelinks,hyperfootnotes=false]{hyperref}

\usepackage{enumitem}

\newcommand{\R}{\mathbb{R}}
\newcommand{\E}{\mathbb{E}}
\newcommand{\rd}{\mathrm{d}}

\newcommand{\bz}{\boldsymbol{z}}
\newcommand{\bW}{\boldsymbol{W}}
\newcommand{\bx}{\boldsymbol{x}}
\newcommand{\bX}{\boldsymbol{X}}
\newcommand{\bXstar}{\bX^{*}}

\newcommand{\cov}{\text{Cov}}
\newcommand{\mean}{\boldsymbol{\mu}}
\newcommand{\covmatrix}{\boldsymbol{\Sigma}}

\newcommand{\state}{\boldsymbol{u}}

\newcommand{\modelnoise}{\boldsymbol{\xi}}
\newcommand{\obs}{\boldsymbol{y}} 
\newcommand{\obsoperator}{\boldsymbol{H}}
\newcommand{\obsmatrix}{H}
\newcommand{\obsnoise}{\boldsymbol{\eta}}
\newcommand{\obsdim}{N_{\obs}}
\newcommand{\obscov}{R}

\newcommand{\scorefun}{\bm{s}}
\newcommand{\interpolantscorefun}{\bar{\bm{s}}}

\newcommand{\drift}{\bm{b}}

\newcommand{\bI}{\bm{I}}

\newcommand{\interpolantobs}{\bar{\obs}}
\newcommand{\interpolantobsnoise}{\bar{\obsnoise}}

\newcommand{\fmvel}{\bm{v}}            
\newcommand{\latentnoise}{\bm{\varepsilon}} 
\newcommand{\gdiff}{g}                  
\newcommand{\genscore}{\bm{s}}          
\newcommand{\src}{\bm{\chi}^{\mathrm{s}}} 
\newcommand{\srccov}{\Sigma^{\mathrm{s}}} 
\newcommand{\srcmean}{\mu^{\mathrm{s}}}   
\newcommand{\sigpath}{\sigma}           

\newcommand{\vscoef}{\kappa}            
\newcommand{\gweight}{w}                

\newcommand{\vorticity}{\omega}
\newcommand{\velocity}{\boldsymbol{v}} 

\newcommand{\statedim}{N_{\state}}

\NewDocumentCommand{\dist}{o o g}{%
  \ensuremath{
    p%
    \IfValueT{#1}{_{#1}}%
    \IfValueT{#2}{^{#2}}%
    \IfValueT{#3}{(#3)}%
  }
}

\NewDocumentCommand{\distbar}{o o g}{%
  \ensuremath{
    \bar{p}%
    \IfValueT{#1}{_{#1}}%
    \IfValueT{#2}{^{#2}}%
    \IfValueT{#3}{(#3)}%
  }
}

\NewDocumentCommand{\conddist}{o o m m}{%
  \ensuremath{
    p%
    \IfValueT{#1}{_{#1}}%
    \IfValueT{#2}{^{#2}}%
    \!\left(#3 \,\middle|\, #4\right)%
  }
}
\NewDocumentCommand{\conddistbar}{o o m m}{%
  \ensuremath{
    \bar{p}%
    \IfValueT{#1}{_{#1}}%
    \IfValueT{#2}{^{#2}}%
    \!\left(#3 \,\middle|\, #4\right)%
  }
}

\NewDocumentCommand{\analysisdist}{o m}{%
\ensuremath{
      \IfValueTF{#1}{%
       \: p^{a}\!\left(#2 \,\middle|\, #1 \right)%
      }{%
        \: p^{a}\!\left(#2 \,\middle|\, \obs\right)%
      }%
  }
}

\NewDocumentCommand{\backgrounddist}{g}{%
  \: p^b\IfValueT{#1}{(#1)}%
}

\NewDocumentCommand{\obsdist}{g}{%
  \: p^{\obs}\IfValueT{#1}{(#1)}%
}

\NewDocumentCommand{\pffdist}{o g}{%
  \ensuremath{
    \IfValueTF{#1}{p_{#1}}{p_{\tau}}%
    \IfValueT{#2}{(#2)}%
  }
}

\hypersetup{
  pdftitle={A Unified Framework for Bayesian Data Assimilation with Generative Models and Observation Interpolants},
  pdfauthor={Nikolaj T. M{\"u}cke and Benjamin Sanderse}
}

\title{A Unified Framework for Bayesian Data Assimilation with Generative Models and Observation Interpolants}

\author{
Nikolaj T. M{\"u}cke$^{*,1,2}$, Benjamin Sanderse$^{1,3}$ \\
\small $^1$Scientific Computing, Centrum Wiskunde \& Informatica \\
\small $^2$Faculty of Civil Engineering and Geosciences, Delft University of Technology \\
\small $^3$Centre for Analysis, Scientific Computing and Applications, Eindhoven University of Technology \\
\texttt{N.T.Mucke-1@tudelft.nl}, \texttt{b.sanderse@cwi.nl}
}

\theoremstyle{plain}
\newtheorem{theorem}{Theorem}[section]
\newtheorem{proposition}[theorem]{Proposition}
\newtheorem{lemma}[theorem]{Lemma}

\theoremstyle{definition}

\theoremstyle{remark}

\begin{document}

\maketitle

\begingroup
\renewcommand\thefootnote{*}
\footnotetext{Corresponding author.}
\endgroup

\begin{abstract}
Bayesian data assimilation combines model forecasts with noisy observations, but sampling high-dimensional, non-Gaussian posteriors remains challenging. We introduce an observation-interpolant framework that turns pretrained stochastic interpolant, flow matching, and diffusion models into posterior samplers without retraining. Conditioning the interpolant path on observations yields a shared likelihood-score correction to the drift or velocity, unifying stochastic and deterministic posterior sampling. The resulting SDEs and ODEs sample the exact posterior when the intermediate likelihood score is known. For practical computation, we approximate this score using a closed-form Gaussian surrogate with a bias-corrected mean and covariance inflated by the model’s source covariance. Jacobian-free and ensemble-shared approximations make the method tractable in high dimensions. We evaluate the framework on linear–Gaussian dynamics, stochastic two-dimensional Navier–Stokes, and urban airflow with up to $\mathcal{O}(10^4)$ degrees of freedom.
\end{abstract}

\section{Introduction}\label{sec:introduction}
Data assimilation (DA) estimates unknown states of dynamical systems by combining prior model predictions with noisy observations. Bayesian DA combines these sources of information in a posterior distribution, which can be approximated by an ensemble of samples. Ensemble Kalman filters provide an efficient approximation but rely on Gaussian assumptions, whereas particle filters are more general but typically require many more samples~\citep{evensen_data_2022}. Neural network surrogates can be used to accelerate the forward model~\citep{mucke_deep_2024, cheng_generalised_2022, chen_reduced-order_2023}, but the learned surrogates themselves are typically deterministic and therefore require a priori model-error distributions that are hard to specify for high-dimensional and nonlinear stochastic systems.

Generative models offer an alternative, by learning to sample directly from a distribution of interest. Two flow-based families dominate: stochastic interpolants (SIs)~\citep{albergo_stochastic_2023, albergo_building_2023, chen_probabilistic_2024}, whose probability paths admit both ODE and SDE samplers, and flow matching (FM)~\citep{lipman_flow_2023}, which is typically sampled through a deterministic ordinary differential equation (ODE). Denoising diffusion~\citep{ho_denoising_2020, song_score-based_2021} corresponds to SDE sampling of an FM-type model. In DA, and posterior sampling in general, they are trained on the prior and modified at sampling time to incorporate observations. All three are increasingly applied to physical forecasting, yet methods for turning them into \emph{posterior} samplers have been developed model by model, each tied to whether the generator is stochastic or deterministic.

\textbf{A unified view.} For data assimilation, we view stochastic interpolants, flow matching, and denoising diffusion as different samplers within one posterior-sampling framework, rather than distinct methods. Each can be described by an interpolant that connects a tractable source distribution to the transition prior $\conddist{\bx_1}{\bx_0}$~\citep{albergo_stochastic_2023, lipman_flow_2023, ma_sit_2024, holderrieth_generator_2024}. Given a pretrained model that samples from this prior, we modify its sampling dynamics to target $\conddist{\bx_1}{\obs,\bx_0}$ without retraining. Conditioning the interpolant on observations gives a common likelihood-score correction to its velocity or drift (Proposition~\ref{prop:conditional_path}). Different choices of diffusion coefficient yield deterministic or stochastic samplers. Adjusting the drift accordingly preserves the conditioned interpolant's distribution at each pseudo-time, so all resulting samplers target the same posterior. In practice, one closed-form Gaussian surrogate approximates the likelihood score for all three. The samplers differ only in their diffusion coefficient and likelihood-score weighting (Figure~\ref{fig:visual_abstract}).

\begin{figure}
    \centering
    \includegraphics[width=1.0\linewidth]{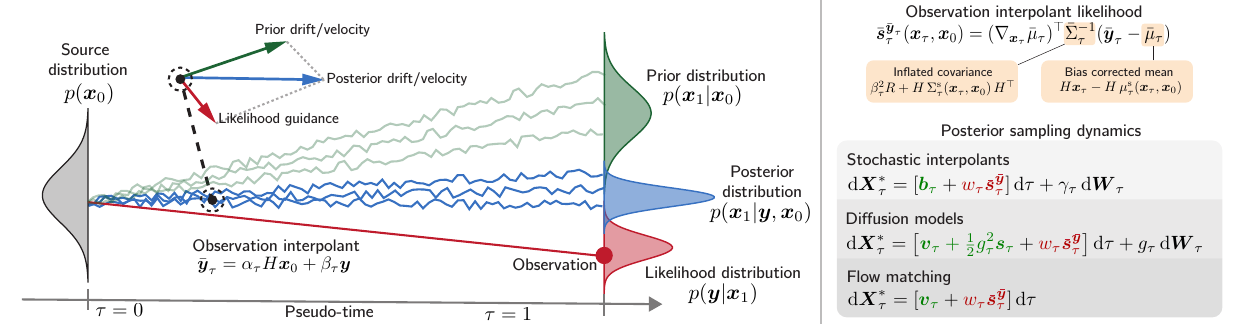}
    \caption{\textbf{Unified posterior sampling for stochastic interpolants, flow matching, and diffusion models.} \emph{Left:} interpolating the observation yields closed-form guidance that steers the pretrained prior flow $p(\bm x_1\mid\bm x_0)$ toward the posterior $p(\bm x_1\mid\bm y,\bm x_0)$, without retraining. \emph{Right:} the weighted guidance $\gweight_\tau$ shifts the prior drift to the posterior drift.}
\label{fig:visual_abstract}
\end{figure}

\textbf{Contributions.} Our contributions are summarized as follows:
\begin{itemize}
    \item \textbf{Unified posterior-sampler family.} One guidance term with weight $\gweight_\tau=\vscoef_\tau+\tfrac12\gdiff_\tau^2$ turns any interpolant-based generative model into a family of posterior samplers.
    \item \textbf{Observation-interpolant likelihood surrogate.} We interpolate the observation with the same schedules as the state and derive its exact conditional mean and covariance, including source-noise bias and covariance inflation. A Gaussian approximation then gives a common likelihood-score surrogate for all three samplers.
    \item \textbf{Scalable covariance approximations.} A Jacobian-free covariance and an ensemble-shared Jacobian make the sparse-observation regime feasible at field scale.
\end{itemize}

\section{Related work}\label{sec:related_work}

\textbf{Flow-based generative models in physics applications.} Flow-based generative models are increasingly used for probabilistic physical forecasting, sampling the next state given the current one via noise around neural-operator predictions~\citep{lippe_pde-refiner_2023}, latent autoregressive dynamics with a diffusion decoder~\citep{gao_generative_2024}, or direct conditional sampling with denoising diffusion~\citep{kohl_benchmarking_2024}, flow matching~\citep{fotiadis_stochastic_2024}, and stochastic interpolants~\citep{chen_probabilistic_2024, mucke_physics-aware_2025}.

\textbf{Posterior sampling and data assimilation with generative models.} Generative models are increasingly used to solve inverse problems~\citep{daras_survey_2024}. A common approach is to train a flow-based model to sample from the prior and, at inference time, condition its dynamics on sparse and noisy observations by modifying the velocity or drift through the score decomposition, $\nabla\log p_\tau(\bx\mid\obs)=\nabla\log p_\tau(\bx)+\nabla\log p_\tau(\obs\mid\bx)$. This principle underlies posterior-sampling methods based on diffusion models~\citep{song_score-based_2021,chung_diffusion_2023,bao_ensemble_2024,qu_deep_2024,rozet_score-based_2023}, flow matching~\citep{pokle_training-free_2024,yan_fig_2024,martin_pnp-flow_2025}, and stochastic interpolants~\citep{chen_flowdas_2025}. These methods differ mainly in how they approximate the likelihood score. Most use a Tweedie-type estimate of the terminal sample and differentiate through it to obtain the score~\citep{song_score-based_2021,rozet_score-based_2023,qu_deep_2024,pokle_training-free_2024,martin_pnp-flow_2025}, while others combine related approximations with Monte Carlo estimation~\citep{chen_flowdas_2025}. They also differ in their treatment of the likelihood covariance, using either the original covariance~\citep{chung_diffusion_2023} or an inflated one~\citep{song_pseudoinverse_2023}. Closest to our work is~\citet{yan_fig_2024}, who interpolate the observations within a flow-matching formulation. In contrast, we develop a unified framework for diffusion models, flow matching, and stochastic interpolants. We interpolate the observations using the same schedules as the state, rather than evaluating the original likelihood at a single estimated terminal state. Deriving the conditional source-noise moments then gives a bias-corrected mean and inflated covariance for a common Gaussian likelihood surrogate across all three samplers.

\section{Preliminaries}\label{sec:preliminaries}

\paragraph{Bayesian inverse problems and data assimilation}
For $n=1,\ldots,N$, consider the Markov model~\citep{carrassi_data_2018, evensen_data_2022}:
\begin{align}\label{eq:data_assimilation_problem}
    \state^{n} \sim \conddist{\state^{n}}{\state^{n-1}},  \quad 
    \obs^{n} = \obsmatrix\state^{n} + \obsnoise^{n},
    \qquad \obsnoise^{n}\sim \dist[\obsnoise],
\end{align}
where $\state^n\in\R^{\statedim}$ is the state, $\obs^n\in\R^{\obsdim}$ the observation, and $\obsmatrix\in \R^{\obsdim\times\statedim}$ the observation matrix. The transition prior and noise-induced likelihood define the \emph{one-step posterior}:
\begin{align}\label{eq:posterior_distribution}
    \conddist{\state^{n}}{\obs^n,\state^{n-1}}
    = \frac{\conddist{\obs^{n}}{\state^{n}}\conddist{\state^{n}}{\state^{n-1}}}
           {\conddist{\obs^{n}}{\state^{n-1}}}.
\end{align}
Feeding each posterior sample into the next step gives an autoregressive assimilation procedure.

\paragraph{Interpolation-based generative modelling}
\label{subsec:flow_models}

To model the transition prior, let $\bx_0=\state^{n-1}$ be the conditioning state and $\bx_1=\state^n$ the target. We use the general interpolant
\begin{align}
\label{eq:general_interpolant}
    \bx_\tau
    =\alpha_\tau\bx_0+\beta_\tau\bx_1+\sigpath_\tau\latentnoise,
    \qquad
    \latentnoise\sim\mathcal{N}(0,\bI),
    \qquad \tau\in[0,1],
\end{align}
where $\tau$ is pseudo-time, $\latentnoise$ is independent of $(\bx_0,\bx_1)$, and
$\beta_0=0$, $\beta_1=1$, and $\alpha_1=\sigpath_1=0$.
The schedule $\sigpath_\tau$ specifies the noise standard deviation conditional on both endpoints.

\textit{Flow matching (FM)} typically sets $\alpha_\tau\equiv0$ and $\sigpath_0=1$, connecting a standard Gaussian source to $\conddist{\bx_1}{\bx_0}$.
Its marginal velocity is learned through~\citep{lipman_flow_2023,liu_flow_2023}
\begin{align}
\label{eq:velocity_loss}
    \mathcal{L}_{\mathrm{vel}}(\theta)
    =
    \E_{\tau,\bx_0,\bx_1,\latentnoise}\left[
        \left\|
            \fmvel_\tau^\theta(\bx_\tau,\bx_0)
            -\left(\dot\beta_\tau\bx_1+\dot\sigpath_\tau\latentnoise\right)
        \right\|^2
    \right],
\end{align}
using data pairs $(\bx_0,\bx_1)$, $\tau\sim\mathcal{U}(0,1)$, and Gaussian noise.
Sampling then consists of drawing $\bX_0$ from the source and integrating the ODE from $\tau=0$ to $\tau=1$:
\begin{align}
\label{eq:fm_ode}
    \text{FM-ODE:}\qquad
    \rd\bX_\tau
    =\fmvel_\tau(\bX_\tau,\bx_0)\,\rd\tau,
    \qquad
    \bX_0\sim\mathcal{N}(0,\bI).
\end{align}
With the exact velocity and well-posed dynamics, $\bX_1\sim\conddist{\bx_1}{\bx_0}$.
We omit the learned parameters $\theta$ below.

\textit{Diffusion models (DM)} typically learn the marginal score
$\genscore_{\tau}(\bx,\bx_0):=\nabla_{\bx}\log\conddist[\tau]{\bx}{\bx_0}$
on Gaussian-source paths~\citep{ho_denoising_2020,song_score-based_2021}. However, for the general interpolant~\eqref{eq:general_interpolant}, the score and velocity satisfy
\begin{align}
\label{eq:general_score_velocity}
    \fmvel_\tau(\bx,\bx_0)
    &=
    \frac{\dot\beta_\tau}{\beta_\tau}\bx
    +\left(
        \dot\alpha_\tau-\frac{\dot\beta_\tau}{\beta_\tau}\alpha_\tau
    \right)\bx_0
    +\vscoef_\tau\genscore_\tau(\bx,\bx_0),
    \qquad
    \vscoef_\tau
    :=
    \frac{\dot\beta_\tau}{\beta_\tau}\sigpath_\tau^2
    -\sigpath_\tau\dot\sigpath_\tau,
\end{align}
for differentiable schedules with $\beta_\tau,\sigpath_\tau>0$ on $(0,1)$ (Appendix~\ref{appendix:score_velocity}).
Thus, FM and DM are equivalent, and a velocity trained via flow matching can be used to sample via the SDE:
\begin{align}
\label{eq:dm_sde}
    \text{DM-SDE:}\quad
    \rd\bX_\tau
    &=
    \left[
        \fmvel_\tau(\bX_\tau,\bx_0)
        +\tfrac12\gdiff_\tau^2\genscore_\tau(\bX_\tau,\bx_0)
    \right]\rd\tau
    +\gdiff_\tau\,\rd\bW_\tau,
    \quad
    \bX_0\sim\mathcal{N}(0,\bI).
\end{align}
Here, $\bW_\tau$ is standard Brownian motion, and $\gdiff_\tau>0$ is the diffusion term.
Unlike $\sigpath_\tau$, $\gdiff_\tau$ controls sampling dynamics rather than interpolant marginals; setting $\gdiff_\tau=0$ recovers the FM ODE. Our labels FM-ODE and DM-SDE distinguish these sampling choices, independently of which field was trained.

\textit{Stochastic interpolants (SI)} can also connect the conditioning state directly to the transition prior.
Following~\citet{chen_probabilistic_2024}, we set $\alpha_0=1$ and $\sigpath_0=0$ and replace the Gaussian noise in~\eqref{eq:general_interpolant} by the Brownian source $\gamma_\tau\bW_\tau$.
Since $\bW_\tau\sim\mathcal{N}(0,\tau\bI)$, this preserves the interpolant's one-time marginals when $\sigpath_\tau=\gamma_\tau\sqrt{\tau}$. In this setting, one can learn the SDE drift directly via:
\begin{align}
\label{eq:si_drift_loss}
    \mathcal{L}_{\mathrm{drift}}(\theta)
    =
    \E_{\tau,\bx_0,\bx_1,\latentnoise}\left[
        \left\|
            \drift_\tau^\theta(\bx_\tau,\bx_0)
            -
            \left(
                \dot\alpha_\tau\bx_0
                +\dot\beta_\tau\bx_1
                +\dot\gamma_\tau\sqrt{\tau}\,\latentnoise
            \right)
        \right\|^2
    \right].
\end{align}
Here, $\dot\gamma_\tau\sqrt{\tau}
=\dot\sigpath_\tau-\sigpath_\tau/(2\tau)$, so this objective is fully determined by the schedules in~\eqref{eq:general_interpolant}.
Sampling proceeds by integrating
\begin{align}
\label{eq:si_sde}
    \text{SI-SDE:}\qquad
    \rd\bX_\tau
    =\drift_\tau(\bX_\tau,\bx_0)\,\rd\tau
    +\gamma_\tau\,\rd\bW_\tau,
    \qquad
    \bX_0=\bx_0,
\end{align}
from $\tau=0$ to $\tau=1$. Note that $\gamma_\tau$ is natively chosen as the diffusion term. The SI drift can be decomposed into velocity and score, $\drift_\tau=\fmvel_\tau+\tfrac12\gamma_\tau^2\genscore_\tau$.
It therefore shares the velocity--score structure of the preceding samplers.

\section{Unified Posterior Sampler}
\label{subsec:posterior_dynamics}

We use the general interpolant of Section~\ref{subsec:flow_models} to sample the one-step posterior $\conddist{\bx_1}{\obs,\bx_0}$ without retraining the prior model by conditioning the drift/velocity on the observations.

Define the intermediate likelihood
\begin{align}\label{eq:likelihood_tilt}
    \Phi_\tau^{\obs}(\bx,\bx_0)
    := \conddist{\obs}{\bX_\tau=\bx,\bx_0}
    = \E\!\left[\conddist{\obs}{\bx_1}\,\middle|\,\bX_\tau=\bx,\bx_0\right],
\end{align}
and its score $\scorefun_\tau^{\obs}:=\nabla_{\bx}\log\Phi_\tau^{\obs}$.
Conditioning preserves the interpolant's schedules and anchor, while shifting its score and velocity by $\scorefun_\tau^{\obs}$ and $\vscoef_\tau\scorefun_\tau^{\obs}$, respectively (Proposition~\ref{prop:conditional_path}).
Applying marginal-preserving ODE/SDE dynamics to this conditioned path~\citep{albergo_stochastic_2023, ma_sit_2024} yields the following unified posterior sampler:

\begin{theorem}[Unified posterior sampler]\label{theorem:unified_posterior}
Let the prior model define an interpolant path with velocity $\fmvel_\tau$ and score $\genscore_\tau$ whose terminal law is $\conddist{\bx_1}{\bx_0}$, and let $\Phi_\tau^{\obs}$ be the intermediate likelihood \eqref{eq:likelihood_tilt}. Assume the posterior-path density $p_\tau^{*}$ is positive and $C^{2}$ on $(0,1)\times\R^{\statedim}$. Let $\gdiff_\tau\ge0$ be a chosen sampler diffusion, equal to $\gamma_\tau$ for the native SI sampler and zero for the FM ODE, measurable with $\gdiff\in L^2([0,1])$, and assume that the SDE below and its Fokker--Planck equation are well posed with unique solutions for the stated initial law. Then
\begin{align}\label{eq:unified_posterior_drift}
    \rd\bXstar_\tau
    = \underbrace{\left[\fmvel_\tau(\bXstar_\tau,\bx_0) + \tfrac12\gdiff_\tau^2\,\genscore_\tau(\bXstar_\tau,\bx_0)\right]}_{\text{prior drift }\drift_\tau}\rd\tau
    + \underbrace{\big(\vscoef_\tau + \tfrac12\gdiff_\tau^2\big)}_{=:\,\gweight_\tau}\scorefun_\tau^{\obs}(\bXstar_\tau,\bx_0)\,\rd\tau
    + \gdiff_\tau\,\rd\bW_\tau ,
\end{align}
initialised from the source law $p_0^{\obs}=p_0$ (unchanged by conditioning because $\beta_0=0$), has marginals $p_\tau^{\obs}(\bx\mid\bx_0)\propto\Phi_\tau^{\obs}(\bx,\bx_0)\,p_\tau(\bx)$ for all $\tau\in[0,1]$, and in particular generates $\bXstar_1\sim\conddist{\bx_1}{\obs,\bx_0}$.
\end{theorem}
See Appendix~\ref{appendix:proof_conditioning} for a proof. While analogous results exist for individual model families, to our knowledge this is the first unified posterior-sampling result explicitly covering stochastic interpolants, flow matching, and diffusion models. As a result, given a trained drift, $\drift_\tau$, velocity, $\fmvel_\tau$, or score $\scorefun_\tau$, one can sample from the posterior distribution via one of the following ODEs/SDEs:
\begin{align*}
    \text{SI-SDE }(\gdiff_\tau=\gamma_\tau):\quad
    \rd\bXstar_\tau
    &= \left[\drift_\tau+\gweight_\tau\scorefun_\tau^{\obs}\right]\rd\tau
       +\gamma_\tau\,\rd\bW_\tau,
    &\quad \bXstar_0 &= \bx_0, \\
    \text{DM-SDE }(\gdiff_\tau>0):\quad
    \rd\bXstar_\tau
    &= \left[\fmvel_\tau+\tfrac12\gdiff_\tau^2\genscore_\tau
       +\gweight_\tau\scorefun_\tau^{\obs}\right]\rd\tau
       +\gdiff_\tau\,\rd\bW_\tau,
    & \bXstar_0 &\sim \mathcal{N}(0,\bI), \\
    \text{FM-ODE }(\gdiff_\tau=0):\quad
    \rd\bXstar_\tau
    &= \left[\fmvel_\tau+\vscoef_\tau\scorefun_\tau^{\obs}\right]\rd\tau,
    & \bXstar_0 &\sim \mathcal{N}(0,\bI).
\end{align*}

\subsection{A closed-form likelihood score from observation interpolation}
\label{subsec:obs_interpolation}

The posterior drift in Eq.~\eqref{eq:unified_posterior_drift} depends on the likelihood score $\scorefun_\tau^{\obs}$ at every intermediate state, yet the conditional expectation that defines it in Eq.~\eqref{eq:likelihood_tilt} is intractable. Estimating it by Monte Carlo or by differentiating through simulated trajectories becomes prohibitive when repeated at every sampler step. We remove this barrier by interpolating the observation alongside the state, so that at every pseudo-time the sampler conditions on an observation matched to its current state: 
\begin{align} \label{eq:interpolated_observation}
    \interpolantobs_\tau(\bx_0, \obs) = \alpha_\tau \obsmatrix\bx_0 + \beta_\tau \obs,
    \qquad \interpolantobs_0 = \alpha_0\obsmatrix\bx_0,\quad \interpolantobs_1 = \obs.
\end{align}
which reduces to $\beta_\tau\obs$ for FM and DM ($\alpha_\tau\equiv0$). This scales the measured endpoint contribution by $\beta_\tau$ and includes the source contribution $\alpha_\tau\obsmatrix\bx_0$. We can then replace the likelihood score from ~\eqref{eq:unified_posterior_drift} with the interpolated likelihood score, $\interpolantscorefun_\tau^{\interpolantobs_\tau}:=\nabla_{\bx_\tau}\log\bar{\Phi}_{\tau}^{\interpolantobs_\tau} = \nabla_{\bx_\tau}\log\conddist{\interpolantobs}{\bX_\tau=\bx,\bx_0}$.

Linearity of $\obsmatrix$ lets us write $\interpolantobs_\tau$ as $\obsmatrix\bx_\tau$ plus a residual containing both observation and source noise. Its exact conditional moments follow from the prior score and its derivative. They differ from the naive mean $\obsmatrix\bx_\tau$ and covariance $\beta_\tau^2\obscov$ because the source noise is correlated with the intermediate state:
\begin{lemma}[Interpolated-observation likelihood]\label{lem:interpolated_likelihood}
Let $\obsoperator$ be linear with matrix $\obsmatrix$ and $\obsnoise \sim \mathcal{N}(0,\obscov)$ independent of $(\bx_1,\latentnoise)$ given $\bx_0$. Then, writing the zero-mean source noise process as $\src_\tau := \sigpath_\tau\latentnoise$ reveals:
\begin{equation}\label{eq:ybar_decomp}
  \interpolantobs_\tau = \obsmatrix\bx_\tau + \interpolantobsnoise_\tau,
  \qquad
  \interpolantobsnoise_\tau := \beta_\tau\obsnoise - \obsmatrix\src_\tau,
\end{equation}
and the conditional mean and covariance of $\conddist{\interpolantobs_\tau}{\bx_0, \bx_\tau}$ are
\begin{align}\label{eq:likelihood_moments}
  \bar{\mu}_\tau(\bx_\tau, \bx_0)
  = \obsmatrix\bx_\tau - \obsmatrix\,\srcmean_\tau(\bx_\tau,\bx_0),
      \quad
  \bar{\Sigma}_{\tau}(\bx_\tau, \bx_0)
  = \beta_\tau^2\obscov + \obsmatrix\,\srccov_\tau(\bx_\tau,\bx_0)\,\obsmatrix^\top,
\end{align}
with source moments $\srcmean_\tau := \E[\src_\tau\mid\bx_\tau,\bx_0]$ and $\srccov_\tau := \cov(\src_\tau\mid\bx_\tau,\bx_0)$.
\end{lemma}

\begin{lemma}[Source conditional moments]\label{lem:source_moments}
For the path \eqref{eq:general_interpolant}, the conditional moments of the source process are
\begin{align}\label{eq:source_moments}
    \srcmean_\tau(\bx,\bx_0) = -\sigpath_\tau^2\,\genscore_\tau(\bx,\bx_0),
    \qquad
    \srccov_\tau(\bx,\bx_0) = \sigpath_\tau^2\,\bI + \sigpath_\tau^4\,\nabla_{\bx}\genscore_\tau(\bx,\bx_0).
\end{align}
\end{lemma}

Both proofs are in Appendix~\ref{appendix:proof_interpolant_likelihood}. The source mean and covariance for the SI ($\src_\tau=\gamma_\tau\bW_\tau$, $\sigpath_\tau=\gamma_\tau\sqrt{\tau}$) are computed as follows:
\begin{align*}
    \srcmean_\tau
    &= -\gamma_\tau^2\tau A_\tau\!\left(\beta_\tau\drift_\tau-\bm c_\tau\right),
    &\quad
    \srccov_\tau
    &= \gamma_\tau^2\tau\bI
       +\gamma_\tau^4\tau^2 A_\tau
        \!\left(\beta_\tau \nabla_{\bx}\drift_\tau-\dot\beta_\tau\bI\right),
\end{align*}
$A_\tau$ and $\bm c_\tau$ defined in \eqref{eq:SI_score}. The FM/DM score is recovered from the velocity through \eqref{eq:fm_score}.

\subsection{Practical approximations}
\label{subsec:approximations}
\label{subsec:covariance_approximation}

\textbf{Gaussian likelihood.}
The moments in Lemmas~\ref{lem:interpolated_likelihood}--\ref{lem:source_moments} are exact for an exact prior score, but the conditional likelihood is not necessarily Gaussian. However, we approximate it by a Gaussian with these moments and differentiate through its mean, holding its covariance fixed:
\begin{align}\label{eq:interpolant_score_closed}
    \interpolantscorefun_\tau^{\interpolantobs_\tau}(\bx_\tau,\bx_0)
    = (\nabla_{\bx_\tau}\bar{\mu}_\tau)^\top
      \bar{\Sigma}_\tau^{-1}
      (\interpolantobs_\tau-\bar{\mu}_\tau),
    \qquad
    \bar{\Sigma}_\tau
    = \beta_\tau^2\obscov
      +\obsmatrix\srccov_\tau\obsmatrix^\top.
\end{align}
This closure is exact in the linear--Gaussian regime. Section~\ref{sec:results} evaluates its accuracy empirically. Compared with the naive likelihood $\mathcal{N}(\obsmatrix\bx_\tau,\beta_\tau^2\obscov)$, the mean correction $-\obsmatrix\srcmean_\tau$ removes the expected source contribution, while the additional covariance captures its residual uncertainty. Retaining prior correlations allows observations to inform unobserved degrees of freedom, particularly under sparse sensing.

\textbf{Jacobian-free approximation.}
Dropping the score Jacobian consistently from the mean Jacobian and covariance gives
$\bar{\Sigma}_\tau
    \approx \beta_\tau^2\obscov
    +\sigpath_\tau^2\obsmatrix\obsmatrix^\top$ and
$(\nabla_{\bx}\bar\mu_\tau)^\top\approx\obsmatrix^\top$. This approximation is accurate when $\sigpath_\tau^2\|\nabla_{\bx}^2\log p_\tau\|\ll1$ and is best suited to dense, informative observations. The covariance is state-independent and precomputable, requiring only an $\obsdim\times\obsdim$ solve per member beyond prior-model evaluations.

\textbf{Ensemble-shared Jacobian.}
Sparse observations benefit from the prior correlations discarded above. Computing the required Jacobian projections separately costs $\obsdim$ Jacobian--vector products per member per step. We instead evaluate them at the ensemble mean and share the resulting $\obsdim\times\obsdim$ factorisation across all $E$ members, reducing the Jacobian work to $\mathcal{O}(\obsdim)$ network passes and one factorisation per step. Refreshing the shared Jacobian every $k$-th step further reduces its evaluation cost by a factor of $k$, while schedule coefficients are updated at every step.

\textbf{Jacobian damping.}
For stability with strongly nonlinear priors, we introduce $\lambda\in[0,1]$ through the parameterisation-specific covariances
\begin{align}\label{eq:damping_formulas_main}
    \widetilde\srccov{}_\tau^{(\lambda),\mathrm{SI}}
    &= \gamma_\tau^2\tau\bI
       +\gamma_\tau^4\tau^2A_\tau
        \left(\lambda\beta_\tau J_{\drift_\tau}
        -\dot\beta_\tau\bI\right),
    &\quad
    \widetilde\srccov{}_\tau^{(\lambda),\mathrm{FM}}
    &= \sigpath_\tau^2\bI
       +\lambda\sigpath_\tau^4\nabla_{\bx}\genscore_\tau.
\end{align}
For SI, damping affects only the learned drift Jacobian; for FM, it scales the full score-Jacobian correction. Thus $\lambda=0$ recovers the Jacobian-free covariance for FM, whereas SI retains the analytical term $-\gamma_\tau^4\tau^2A_\tau\dot\beta_\tau\bI$.

\subsection{Implementation}
\label{subsec:summary_table}\label{subsec:Implementation}

\begin{figure}
    \centering
    \includegraphics[width=1.0\linewidth]{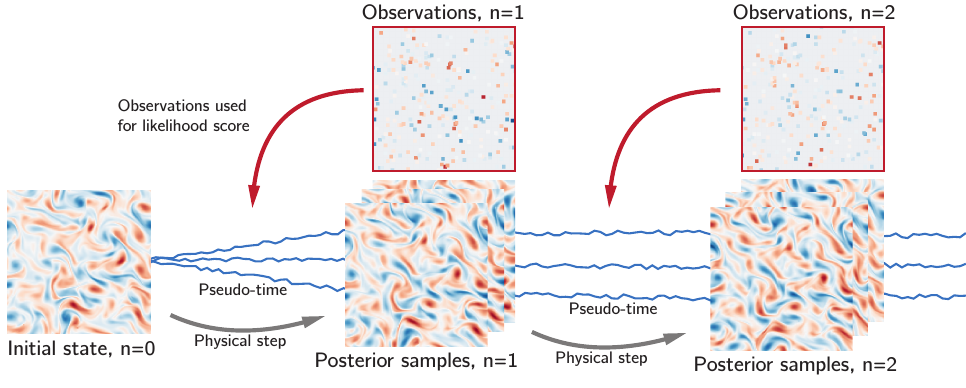}
    \caption{\textbf{Sequential assimilation with the unified sampler.} From a single initial state, $\state^{0}$, an ensemble of posterior samples are generated via the guided
SDE/ODE (Algorithm~\ref{alg:unified}). Each posterior member then initializes the next
forecast, yielding an autoregressive rollout of~\eqref{eq:posterior_distribution}.}
\label{fig:sequential_da}
\end{figure}

All three samplers need only the pre-trained model and an ODE/SDE solver, and share the assimilation loop in Algorithm~\ref{alg:unified}. Recovered coefficients degenerate at $\tau=0$ (in particular, the FM guidance representation is singular because $\beta_0=0$), so the first step uses the prior drift only and the correction is applied from $\tau=\Delta\tau$ onwards. Independent draws give an ensemble. Feeding each posterior sample $\state^{n}$ back before assimilating $\obs^{n+1}$ advances the autoregressive assimilation in physical time $n$ (see Figure \ref{fig:sequential_da}). The solver and number of pseudo-time steps are chosen post-training, here forward Euler for the guided ODE and Euler--Maruyama for the SDEs. Schedules, architectures, and training details are in Appendix~\ref{appendix:implementation}.

\begin{wrapfigure}{r}{0.5\textwidth}
\vspace{-55pt} 
\begin{minipage}{\linewidth}
\setlength{\intextsep}{0pt} 
\begin{algorithm}[H]
\caption{Unified posterior assimilation step}
\label{alg:unified}
\begin{algorithmic}[1]
\STATE \textbf{Input:} $\bx_0=\state^{n-1}$, $\obs=\obs^n$, trained model, schedules, diffusion $\gdiff_\tau$, steps $M$
\STATE Initialise $\bXstar$ from sampler source
\STATE Set $\Delta\tau=1/M$
\FOR{$\tau=0,\Delta\tau,\ldots,1-\Delta\tau$}
    \STATE Form prior drift $\drift_\tau$ (use native $\drift_\tau$ for SI; otherwise $\fmvel_\tau+\tfrac12\gdiff_\tau^2\genscore_\tau$)
    \IF{$\tau=0$}
        \STATE Set likelihood correction $\bm r_\tau=\bm 0$
    \ELSE
        \STATE Compute $\interpolantobs_\tau=\alpha_\tau\obsmatrix\bx_0+\beta_\tau\obs$
        \STATE Compute source moments $\srcmean_\tau, \srccov_\tau$
        \STATE Set $\bar\mu_\tau=\obsmatrix\bXstar-\obsmatrix\srcmean_\tau$
        \STATE Set $\bar\Sigma_\tau=\beta_\tau^2\obscov+\obsmatrix\srccov_\tau\obsmatrix^\top$
        \STATE Compute $\interpolantscorefun_\tau^{\interpolantobs_\tau}=(\nabla_{\bx}\bar\mu_\tau)^\top\bar\Sigma_\tau^{-1}(\interpolantobs_\tau-\bar\mu_\tau)$
        \STATE Set $\bm r_\tau=(\vscoef_\tau+\tfrac12\gdiff_\tau^2)\interpolantscorefun_\tau^{\interpolantobs_\tau}$
    \ENDIF
    \STATE Draw noise $\bz\sim\mathcal{N}(0,\bI)$
    \STATE $\bXstar\leftarrow\bXstar+\left(\drift_\tau+\bm r_\tau\right)\Delta\tau+\gdiff_\tau\sqrt{\Delta\tau}\,\bz$
\ENDFOR
\STATE \textbf{Return:} $\state^n=\bXstar$
\end{algorithmic}
\end{algorithm}
\end{minipage}
\vspace{-10pt} 
\end{wrapfigure}

\section{Results}\label{sec:results}

We evaluate three increasingly realistic problems: a linear--Gaussian system with an exact posterior, stochastically forced 2D Navier--Stokes, and a 2D slice of a 3D urban airflow. All cases use synthetic data, which means that we have direct access to the full ground truth states for evaluation. The samplers of Section~\ref{subsec:summary_table} are reported as \emph{Ours SI-SDE}, \emph{DM-SDE}, and \emph{FM-ODE}, using either a Jacobian-free approximation (\emph{Jac-free}) or a shared Jacobian updated every tenth pseudo-step (\emph{shared}). The field experiments are assimilated autoregressively with ensembles of 64 members.

\textbf{Baselines, scenarios, and metrics.} We compare with FlowDAS~\citep{chen_flowdas_2025}, SDA~\citep{rozet_score-based_2023}, D-Flow SGLD~\citep{parikh_d-flow_2026,ben-hamu_d-flow_2024}, Guided FM~\citep{yan_fig_2024}, and SURGE~\citep{wei_surge_2026}. All methods use the same prior trained SI or FM (DM is derived from the trained FM) to isolate the effect of the posterior sampling techniques rather than the quality of the training. See Appendix \ref{appendix:implementation} for details on the neural network architecture and training. According to \cite{wei_surge_2026}, both SDA and FlowDAS perform better when combined with SURGE, so we report only the SURGE-enhanced versions. For Navier--Stokes, we additionally include an ensemble Kalman filter~\citep{evensen_data_2022} and bootstrap particle filter~\citep{carrassi_data_2018} using the true solver. Hyperparameters are tuned on separate short trajectories. Details are given in Appendix~\ref{appendix:methods}.

We consider super-resolution from $32^2$ or $16^2$ to $128^2$ (where the observation operator is simply subsampling grid points), and random sensors covering $5\%$ or $1.5625\%$ of the domain for the Navier--Stokes case and $1.5625\%$ or $0.78125\%$ for the urban case, with Gaussian observation noise. Navier--Stokes uses all four settings and urban airflow the two sparse-sensor settings. For the field cases, we report RMSE, CRPS, and runtime per assimilation step, and additional metrics in the appendix. For the linear--Gaussian case, we report Gaussian KL divergence to the exact posterior. Metric definitions are given in Appendix~\ref{appendix:metrics}. For the Navier-Stokes and the Urban test case, the metrics are computed and aggregated over several autoregressive posterior sampling steps. All runs are performed on an NVIDIA GeForce RTX 3090 with 24 GB memory.

\subsection{Simple analytical test case}
\label{subsec:results_analytical}

\begin{wrapfigure}[16]{r}{0.5\linewidth}
\vspace{-50pt} 
\centering
\includegraphics[width=\linewidth]{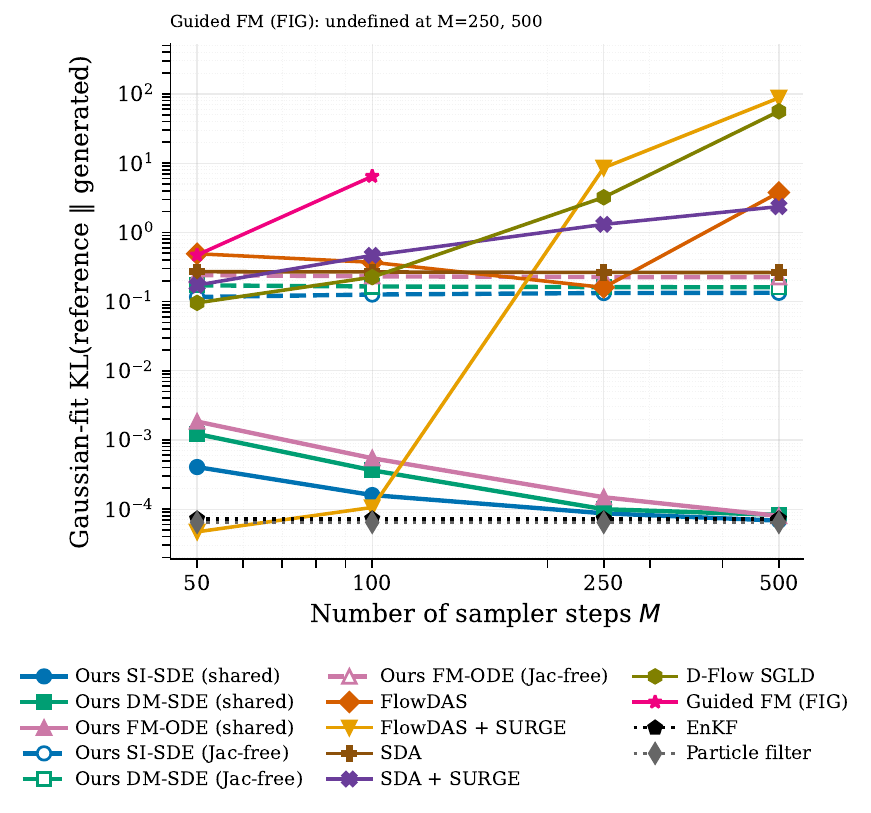}
\vspace{-20pt} 
\caption{KL to the exact posterior vs.\ sampler steps $M$ (five seeds).}
\label{fig:analytical}
\end{wrapfigure}

The first case is a 2D linear--Gaussian system with $\statedim=\obsdim=2$ of Appendix~\ref{appendix:analytical_case}: $\state^1 = \state^{0} + \modelnoise$, $\obs^1 = \obsmatrix\state^1 + \obsnoise$ with standard-normal noises and $\obsmatrix=\obscov=\bI$. The exact posterior, SI drift, and FM velocity are known in closed form, so no network is trained and the samplers are isolated from model error. We use 100,000 samples for all methods. 

Figure~\ref{fig:analytical} shows the KL divergence to the exact posterior versus the number of sampler steps $M$. With the full inflated covariance, all three of our samplers converge to the exact posterior within sampling constraints (KL $\approx 10^{-4}$), matching the EnKF and particle filter in their ideal linear--Gaussian regime; the Jacobian-free variant performs significantly worse (KL $\approx 0.12$--$0.24$). Among the baselines, only FlowDAS+SURGE also recovers the exact posterior. However, it is worth noting that with increasing number of SDE steps, it diverges, suggesting high sensitivity to hyperparameters. Guidance alone does not converge at any step count (SDA KL $\approx 0.27$, FlowDAS $\geq 0.49$), Guided FM collapses, and SDA+SURGE and D-Flow SGLD perform modestly at moderate $M$ but diverge as $M$ grows at fixed per-step hyperparameters. Density panels and a full table of results are in Appendix~\ref{appendix:analytical_case}.

\subsection{Stochastic Navier--Stokes}
\label{subsec:results_ns}
The second case is an incompressible fluid on the torus $\mathcal{T}^2=[0,2\pi]^2$, governed by the stochastically forced two-dimensional vorticity equation,
\begin{align}\label{eq:ns_equation}
    \rd\vorticity + \velocity\cdot\nabla\vorticity\,\rd t
    = \nu\Delta\vorticity\,\rd t - \alpha\vorticity\,\rd t + \varepsilon\,\rd\xi,
\end{align}
where $\velocity=\nabla^\perp\psi$ with $-\Delta\psi=\vorticity$ and $\rd\xi$ a temporally white forcing on four Fourier modes. Trajectories are generated with JAX-CFD~\citep{jax_cfd_2021} on a $256^2$ pseudo-spectral grid at $\nu=10^{-3}$ and down-sampled to $128^2$; we assimilate 15 consecutive steps at interval $\Delta t=0.5$, observing vorticity through $\obs=\obsoperator(\vorticity)+\obsnoise$ with $\obsnoise\sim\mathcal{N}(0,0.05^2\bI)$ (data generation details in Appendix~\ref{appendix:ns_case}).

\begin{figure*}[t]
    \centering
    \setlength{\tabcolsep}{1.5pt}
    \renewcommand{\arraystretch}{0.8}
    \resizebox{\linewidth}{!}{%
    \begin{tabular}{@{}c@{\hspace{5pt}}c@{}}
    \begin{tabular}[c]{@{}c@{}}
        \includegraphics[width=0.15\linewidth]{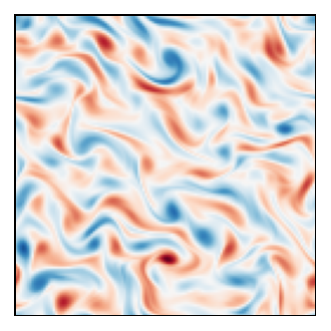} \\[1pt]
        {\scriptsize\textbf{True state}} \\[6pt]
        \includegraphics[width=0.15\linewidth]{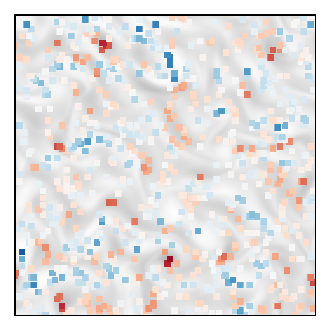} \\[1pt]
        {\scriptsize\textbf{Observations}}
    \end{tabular}
    &
    \begin{tabular}[c]{@{}ccccccc@{}}
        &
        {\scriptsize\textbf{SI-SDE (shared)}} &
        {\scriptsize\textbf{SI-SDE (Jac-free)}} &
        {\scriptsize\textbf{SDA+SURGE}} &
        {\scriptsize\textbf{FlowDAS+SURGE}} &
        {\scriptsize\textbf{D-Flow SGLD}} & \\
        \rotatebox{90}{\makebox[0.15\linewidth][c]{\scriptsize\textbf{Posterior mean}}} &
        \includegraphics[width=0.15\linewidth]{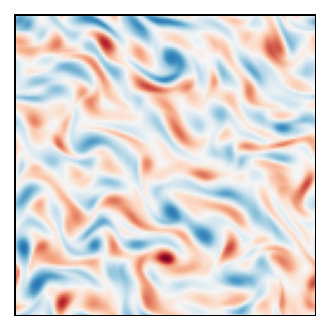} &
        \includegraphics[width=0.15\linewidth]{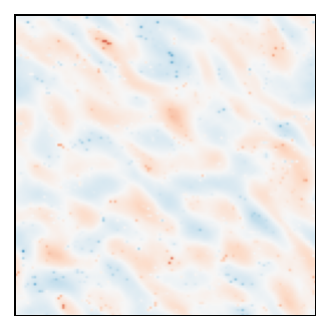} &
        \includegraphics[width=0.15\linewidth]{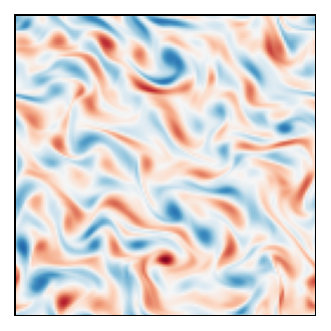} &
        \includegraphics[width=0.15\linewidth]{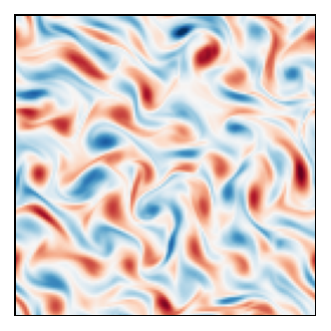} &
        \includegraphics[width=0.15\linewidth]{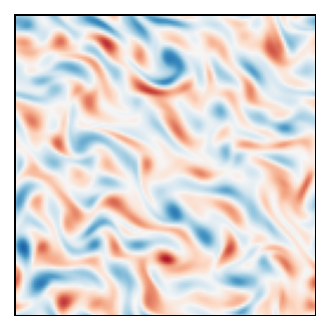} &
        \includegraphics[height=0.15\linewidth]{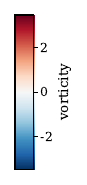} \\
        \rotatebox{90}{\makebox[0.15\linewidth][c]{\scriptsize\textbf{Ensemble std.}}} &
        \includegraphics[width=0.15\linewidth]{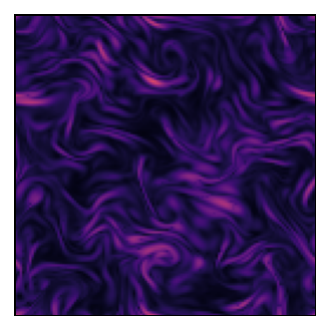} &
        \includegraphics[width=0.15\linewidth]{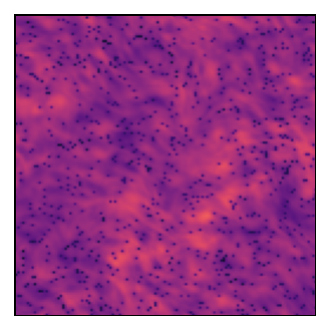} &
        \includegraphics[width=0.15\linewidth]{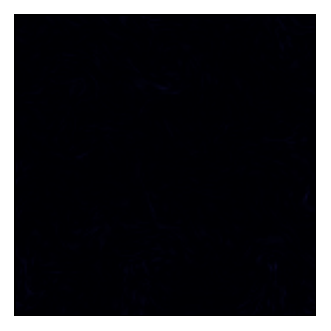} &
        \includegraphics[width=0.15\linewidth]{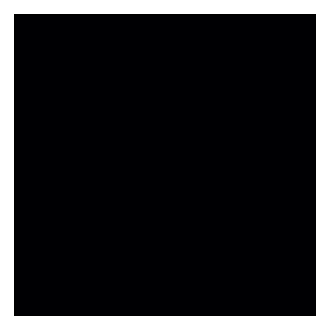} &
        \includegraphics[width=0.15\linewidth]{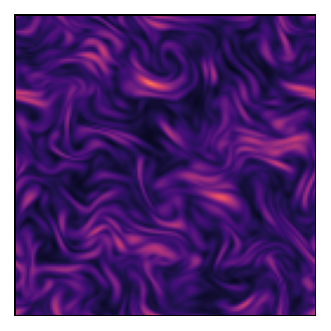} &
        \includegraphics[height=0.15\linewidth]{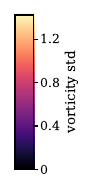} \\
        \rotatebox{90}{\makebox[0.15\linewidth][c]{\scriptsize\textbf{Absolute error}}} &
        \includegraphics[width=0.15\linewidth]{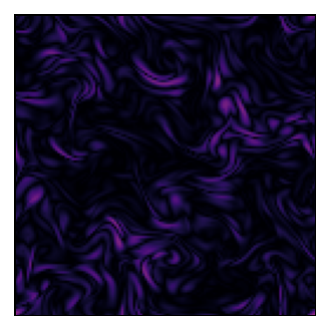} &
        \includegraphics[width=0.15\linewidth]{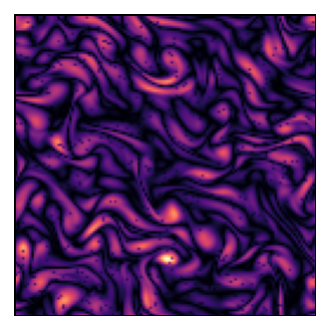} &
        \includegraphics[width=0.15\linewidth]{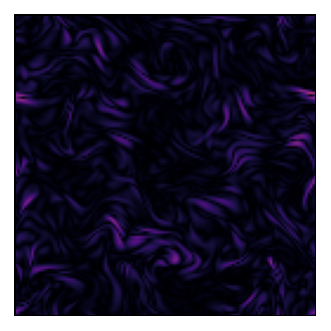} &
        \includegraphics[width=0.15\linewidth]{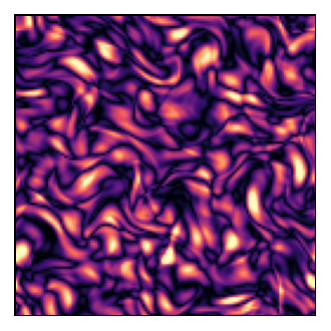} &
        \includegraphics[width=0.15\linewidth]{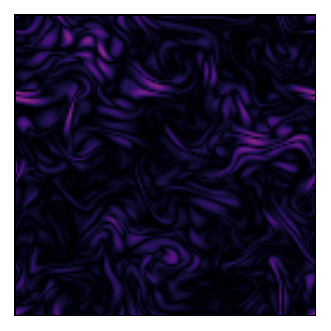} &
        \includegraphics[height=0.15\linewidth]{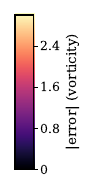} \\
    \end{tabular}
    \end{tabular}}
    \vspace{-10pt}
    \caption{Navier--Stokes, sparse sensors at $5\%$: truth and observations (left), and the posterior mean, ensemble standard deviation, and absolute error at the final assimilation step. One colour scale per row. All methods and scenarios: Figures~\ref{fig:ns_fields_32}--\ref{fig:ns_fields_sparse1p5}.}
    \label{fig:ns_fields_main}
\end{figure*}

Table~\ref{tab:ns_accuracy} summarizes accuracy and cost, while Figure~\ref{fig:ns_fields_main} shows reconstructions for the $5\%$-sensor case. In dense super-resolution settings, our samplers achieve the best RMSE and CRPS; at $32^2$, the Jacobian-free covariance matches the shared approximation at roughly one-tenth the cost. Under sparse observations, however, covariance quality becomes critical: the shared inflated covariance reduces SI-SDE RMSE by $2$--$4\times$ relative to the Jacobian-free variant, which, like Guided FM, over-smooths toward the prior mean. SI-SDE is therefore the strongest of our samplers in this regime. SDA+SURGE and D-Flow SGLD remain competitive, but SDA+SURGE collapses the ensemble, while D-Flow costs ${\approx}700$\,s per step—about an order of magnitude more than SI-SDE (shared) and two orders more than the Jacobian-free samplers. SI-SDE (shared) also achieves the best or tied-best CRPS in three of four scenarios among methods using $E=64$. Classical filters with $E=64$ are not competitive; the $E=1000$ EnKF reference attains the lowest sparse-scenario errors but requires the true solver and ${>}100$\,s per step. Additional convergence, error, and enstrophy results are given in Appendix~\ref{appendix:ns_case}.

\begin{table*}[ht]
    \centering
    \setlength{\tabcolsep}{3.5pt}
    \caption{Navier--Stokes: vorticity RMSE and CRPS across the four observation scenarios, with per-step cost (lower is better). $E=64$ and $M=100$ sampler steps. \textbf{Bold} marks the best value in each column, together with every value equal to it to 2 decimal places in unscaled units.}
    \label{tab:ns_accuracy}
    \vspace{-10pt}
    \resizebox{\ifdim\width>\textwidth\textwidth\else\width\fi}{!}{%
    \begin{tabular}{l|cccc|cccc|cccc}
        \toprule
        & \multicolumn{4}{c|}{\textbf{Vorticity RMSE ($\times 10^{-1}$)}}
        & \multicolumn{4}{c|}{\textbf{Vorticity CRPS ($\times 10^{-1}$)}}
        & \multicolumn{4}{c}{\textbf{Cost (s/step)}} \\
        & $16^2$ & $32^2$ & $5\%$ & $\tfrac{1}{64}$ & $16^2$ & $32^2$ & $5\%$ & $\tfrac{1}{64}$ & $16^2$ & $32^2$ & $5\%$ & $\tfrac{1}{64}$ \\
        \midrule
        \multicolumn{13}{l}{\textit{Ours -- ensemble-shared inflated covariance (Jacobian refreshed every $k=10$ steps)}} \\
        Ours (SI-SDE) & \textbf{2.024} & 0.667 & 1.731 & 3.854 & \textbf{1.062} & \textbf{0.347} & \textbf{0.872} & 2.073 & 41.3 & 77.7 & 68.4 & 41.8 \\
        Ours (DM-SDE) & 2.162 & \textbf{0.638} & 3.424 & 6.834 & 1.154 & \textbf{0.333} & 1.823 & 3.830 & 87.2 & 164.1 & 144.3 & 88.6 \\
        Ours (FM-ODE) & 2.102 & \textbf{0.626} & 3.828 & 7.038 & \textbf{1.128} & \textbf{0.327} & 2.051 & 3.945 & 89.1 & 164.9 & 144.7 & 89.3 \\
        \midrule
        \multicolumn{13}{l}{\textit{Ours -- Jacobian-free covariance}} \\
        Ours (SI-SDE) & 2.765 & 0.670 & 6.969 & 8.500 & 1.518 & 0.355 & 3.952 & 4.818 & \textbf{6.9} & 7.1 & 7.1 & 6.9 \\
        Ours (DM-SDE) & 2.759 & 0.672 & 6.756 & 8.362 & 1.501 & 0.351 & 3.754 & 4.681 & 13.0 & 13.3 & 13.2 & 13.0 \\
        Ours (FM-ODE) & 2.733 & 0.686 & 7.026 & 8.460 & 1.492 & 0.359 & 3.897 & 4.732 & 13.0 & 13.2 & 13.1 & 13.0 \\
        \midrule
        FlowDAS + SURGE & 5.066 & 1.540 & 8.632 & 11.896 & 3.853 & 1.164 & 6.742 & 9.426 & 32.2 & 33.5 & 33.2 & 33.7 \\
        SDA + SURGE & 3.699 & 1.577 & \textbf{1.616} & 2.970 & 2.642 & 1.080 & 1.130 & 2.092 & 64.7 & 68.3 & 69.5 & 69.6 \\
        D-Flow SGLD & 2.455 & 2.010 & 1.974 & \textbf{2.313} & 1.272 & 1.004 & 0.992 & \textbf{1.179} & 671.6 & 693.7 & 701.5 & 704.2 \\
        Guided FM (FIG) & 2.986 & 0.780 & 7.231 & 8.461 & 1.623 & 0.415 & 4.082 & 4.756 & 13.6 & 14.0 & 14.1 & 14.2 \\
        \midrule
        EnKF & 2.683 & 2.161 & 2.768 & 6.222 & 1.680 & 1.320 & 1.384 & 3.319 & 7.5 & 7.5 & 7.5 & 7.5 \\
        Particle filter & 8.753 & 8.388 & 8.240 & 8.577 & 6.914 & 6.613 & 6.478 & 6.751 & 7.3 & \textbf{6.8} & \textbf{6.8} & \textbf{6.8} \\
        \midrule
        EnKF reference ($E=1000$) & 1.442 & 0.941 & 0.935 & 1.053 & 0.779 & 0.511 & 0.516 & 0.565 & 102.0 & 102.4 & 103.0 & 122.2 \\
        \bottomrule
    \end{tabular}}
    \vspace{-15pt}
  \end{table*}

\subsection{Urban airflow}
\label{subsec:results_urban}

The third case considers coupled velocity $(u,v,w)$ and temperature fields over an urban area with several buildings, simulated with the building-resolving LES code uDALES~\citep{suter_udales_2022}. The $128\times128\times128\,\mathrm{m}^3$ domain has $1\,\mathrm{m}$ resolution and is periodic in the horizontal directions. Simulations are spun up for 100~s and sampled every 2~s, retaining a horizontal slice at 2~m height. We initialize from a different held-out flow state and forecast 145 physical steps, assimilating sparse velocity observations every third step (49 updates). White noise with standard deviation $0.05$ is added to all velocity observations. Both velocity and temperature are reconstructed, so temperature is inferred only through its coupling to the observed flow. Table~\ref{tab:urban_accuracy_M50} reports rollout-averaged results at both sensor densities, and Figure~\ref{fig:urban_fields_main} shows temperature reconstructions at $0.78125\%$ density. Further details are in Appendix~\ref{appendix:urban_case}.

\newcommand{\urbmainpanel}[1]{%
  {\includegraphics[width=0.15\linewidth]{figures/urban/singles/urban_states_temperature_sparse_0_78125__#1.pdf}}%
}
\newcommand{\urbmainref}[1]{%
  {\includegraphics[width=0.15\linewidth]{figures/urban/singles/#1.pdf}}%
}
\newcommand{\urbmaincbar}[1]{%
  {\includegraphics[height=0.15\linewidth]{figures/urban/singles/urban_states_temperature_sparse_0_78125__cbar_#1.pdf}}%
}

\begin{figure*}[t]
    \centering
    \setlength{\tabcolsep}{1.5pt}
    \renewcommand{\arraystretch}{0.8}
    \resizebox{\linewidth}{!}{%
    \begin{tabular}{@{}c@{\hspace{5pt}}c@{}}
    \begin{tabular}[c]{@{}c@{}}
        \vspace{-5.0pt}
        \urbmainref{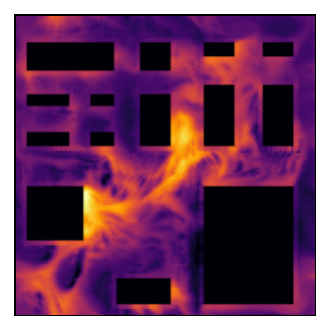} \\[1pt]
        \vspace{-5.0pt}
        {\scriptsize\textbf{True temperature}} \\[6pt]
        \vspace{-5.0pt}
        \urbmainref{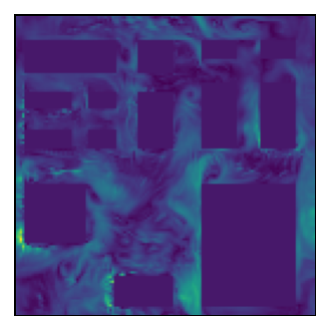} \\[1pt]
        \vspace{-5.0pt}
        {\scriptsize\textbf{True velocity}} \\[6pt]
        \vspace{-5.0pt}
        \urbmainref{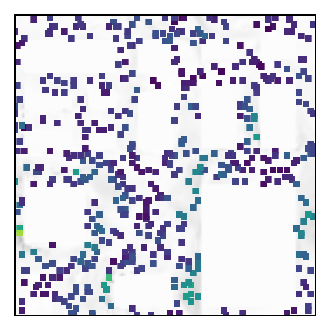} \\[1pt]
        \vspace{-5.0pt}
        {\scriptsize\textbf{Velocity observations}}
    \end{tabular}
    &
    \begin{tabular}[c]{@{}ccccccc@{}}
        &
        {\scriptsize\textbf{SI-SDE (shared)}} &
        {\scriptsize\textbf{SI-SDE (Jac-free)}} &
        {\scriptsize\textbf{SDA+SURGE}} &
        {\scriptsize\textbf{FlowDAS+SURGE}} &
        {\scriptsize\textbf{D-Flow SGLD}} & \\
        \rotatebox{90}{\makebox[0.15\linewidth][c]{\scriptsize\textbf{Posterior mean}}} &
        \urbmainpanel{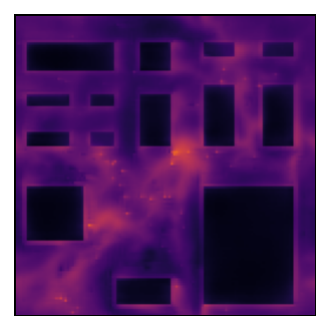} &
        \urbmainpanel{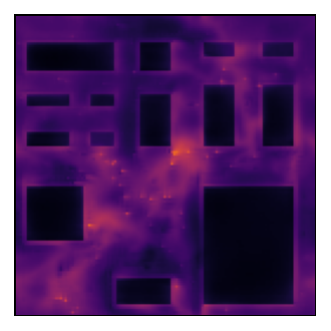} &
        \urbmainpanel{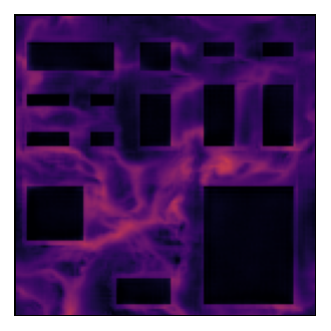} &
        \urbmainpanel{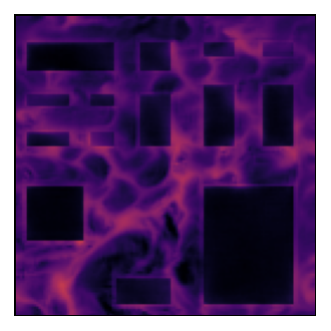} &
        \urbmainpanel{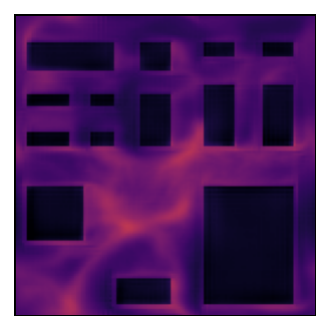} &
        \urbmaincbar{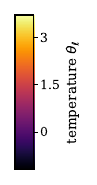} \\
        \rotatebox{90}{\makebox[0.15\linewidth][c]{\scriptsize\textbf{Ensemble std.}}} &
        \urbmainpanel{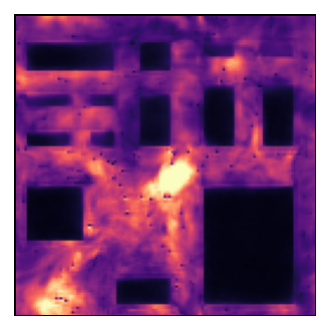} &
        \urbmainpanel{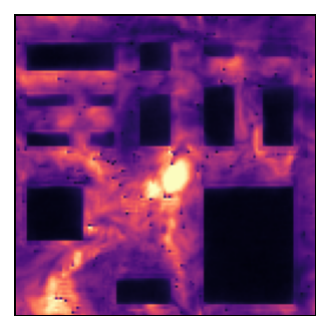} &
        \urbmainpanel{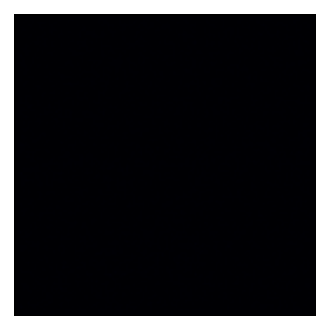} &
        \urbmainpanel{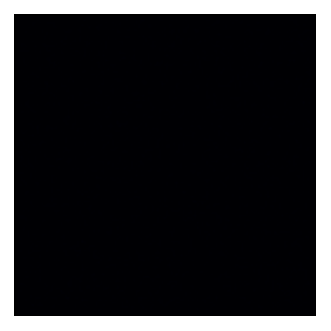} &
        \urbmainpanel{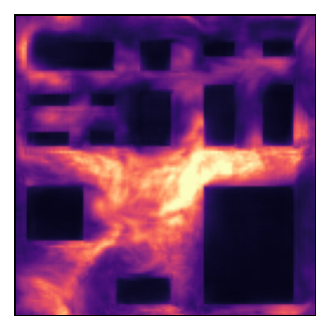} &
        \urbmaincbar{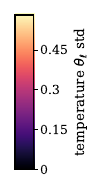} \\
        \rotatebox{90}{\makebox[0.15\linewidth][c]{\scriptsize\textbf{Absolute error}}} &
        \urbmainpanel{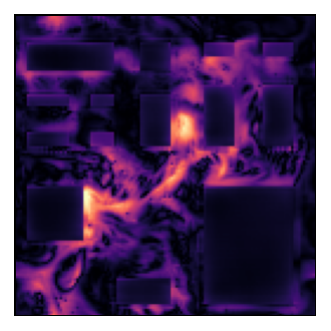} &
        \urbmainpanel{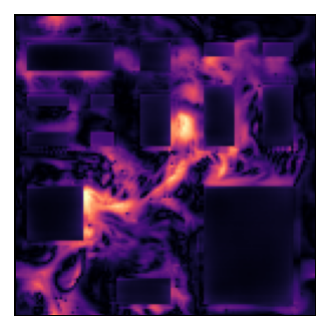} &
        \urbmainpanel{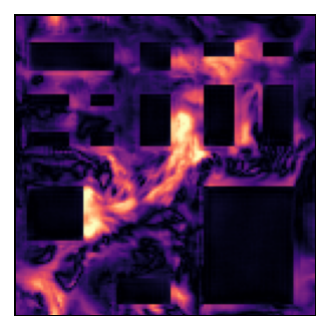} &
        \urbmainpanel{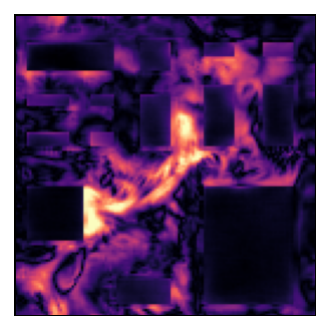} &
        \urbmainpanel{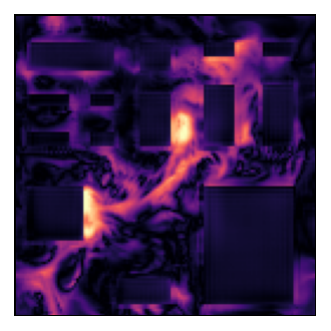} &
        \urbmaincbar{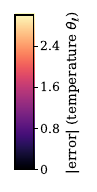} \\
    \end{tabular}
    \end{tabular}}
    \vspace{-5pt}
    \caption{Urban airflow, sparse sensors at $0.78125\%$: true potential temperature, true velocity magnitude, and observed velocity sensors (left), followed by the posterior mean, ensemble standard deviation, and absolute error of the unobserved \emph{temperature} at the final forecast step.}
    \label{fig:urban_fields_main}
    \vspace{-15pt}
\end{figure*}

Table~\ref{tab:urban_accuracy_M50} reveals a split between observed and unobserved variables. The computational cost per step is omitted as timings are approximately the same as for the Navier-Stokes case. DM-SDE gives the lowest velocity RMSE and CRPS and best calibration (see spread-skill figures in Figure~\ref{fig:urban_ss_vs_step}) among our samplers, whereas SI-SDE gives the lowest temperature RMSE and CRPS at both sensor densities. The shared covariance improves SI-SDE temperature RMSE from $0.38/0.43$ to $0.36/0.42$. D-Flow SGLD is the closest baseline, matching the best velocity or CRPS values but costing significantly more per step. Both SURGE ensembles are severely underdispersed. Over the long rollout, velocity errors decrease as the observations correct the mismatched initial state, while errors in the unobserved temperature accumulate (Appendix~\ref{appendix:urban_case}). Figure~\ref{fig:urban_fields_main} nevertheless shows that SI-SDE recovers the broad thermal-plume structure and places uncertainty around unresolved fine scales.

\begin{wraptable}{r}{0.6\textwidth}
    \centering
    \vspace{-40pt}
    \setlength{\tabcolsep}{3.5pt}
    \caption{Urban (uDALES): rollout-averaged velocity RMSE
    (V-RMSE), temperature RMSE (T-RMSE), velocity CRPS
    (V-CRPS), and temperature CRPS (T-CRPS).
    Mean over 4--5 held-out trajectories (cell-dependent),
    $E=64$, and $M=50$ sampler steps.
    \textbf{Bold} marks the best value in each column,
    including ties to 2 decimal places.}
    \vspace{-10pt}
    \label{tab:urban_accuracy_M50}
    \resizebox{\linewidth}{!}{%
    \begin{tabular}{l|cc|cc|cc|cc}
        \toprule
        & \multicolumn{2}{c|}{\textbf{V-RMSE}}
        & \multicolumn{2}{c|}{\textbf{T-RMSE}}
        & \multicolumn{2}{c|}{\textbf{V-CRPS}}
        & \multicolumn{2}{c}{\textbf{T-CRPS}} \\
        & $\tfrac{1}{64}$ & $\tfrac{1}{128}$
        & $\tfrac{1}{64}$ & $\tfrac{1}{128}$
        & $\tfrac{1}{64}$ & $\tfrac{1}{128}$
        & $\tfrac{1}{64}$ & $\tfrac{1}{128}$ \\
        \midrule
        \multicolumn{9}{l}{\textit{Ours -- shared covariance
        (Jacobian refresh $k=10$)}} \\
        Ours (SI-SDE)
        & 0.60 & 0.67 & \textbf{0.36} & \textbf{0.42}
        & 0.33 & 0.38 & \textbf{0.21} & \textbf{0.25} \\
        Ours (DM-SDE)
        & \textbf{0.56} & \textbf{0.62} & 0.45 & 0.55
        & \textbf{0.30} & \textbf{0.34} & 0.28 & 0.36 \\
        Ours (FM-ODE)
        & 0.59 & 0.66 & 0.40 & 0.51
        & 0.33 & 0.37 & 0.25 & 0.33 \\
        \midrule
        \multicolumn{9}{l}{\textit{Ours -- Jacobian-free covariance}} \\
        Ours (SI-SDE)
        & 0.61 & 0.68 & 0.38 & 0.43
        & 0.34 & 0.38 & 0.23 & 0.26 \\
        Ours (DM-SDE)
        & 0.57 & 0.63 & 0.47 & 0.56
        & \textbf{0.30} & \textbf{0.34} & 0.29 & 0.37 \\
        Ours (FM-ODE)
        & 0.60 & 0.66 & 0.42 & 0.52
        & 0.33 & 0.37 & 0.26 & 0.34 \\
        \midrule
        FlowDAS + SURGE
        & 0.88 & 0.89 & 0.53 & 0.55
        & 0.66 & 0.66 & 0.40 & 0.42 \\
        SDA + SURGE
        & 0.73 & 0.80 & 0.47 & 0.55
        & 0.53 & 0.58 & 0.34 & 0.41 \\
        D-Flow SGLD
        & 0.57 & \textbf{0.62} & \textbf{0.36} & 0.43
        & 0.32 & 0.35 & \textbf{0.21} & \textbf{0.25} \\
        Guided FM (FIG)
        & 0.59 & 0.65 & 0.41 & 0.51
        & 0.33 & 0.36 & 0.25 & 0.33 \\
        \bottomrule
    \end{tabular}}
    \vspace{-10pt}
\end{wraptable}

\section{Conclusion}\label{sec:conclusion}
We introduced a unified posterior-sampling framework for stochastic interpolants, flow matching, and diffusion models based on a shared observation-interpolant likelihood correction. With the full covariance, all three samplers closely approximate the exact posterior in the linear--Gaussian test. On Navier--Stokes, our methods achieve the best accuracy under dense observations among all the generative posterior samplers, while the shared covariance substantially improves SI-SDE under sparse observations. In the urban case, DM-SDE performs best on observed velocity, while SI-SDE best limits error growth in temperature inferred from velocity alone. There, the Jacobian-free approximation remains close to the shared variant at $3$--$4\times$ lower cost. Thus the best sampler and covariance approximation are problem-dependent: prior correlations are decisive for sparse Navier--Stokes, but provide only modest gains in the urban case.

\textbf{Limitations.} The exactness result assumes an exact prior and likelihood score, while the practical method uses a Gaussian approximation, stopped covariance gradients, and approximate Jacobians. These approximations can introduce bias, as seen for the Jacobian-free variant in the analytical and sparse Navier--Stokes cases. The current formulation also assumes linear Gaussian observations, and the shared covariance becomes costly as the observation dimension grows. Finally, one-step autoregressive assimilation neither reweights past states nor uses future observations, so model and approximation errors may accumulate over long rollouts. Extending the method to nonlinear observations, smoothing, and broader model mismatch remains future work.

\subsection*{Acknowledgments}
This research was funded by the National Growth Fund of the Netherlands and administered by the Netherlands Organisation for Scientific Research (NWO) under the AINed XS grant NGF.1609.242.037.

\subsection*{Reproducibility statement}
Theoretical claims and assumptions from Section~\ref{subsec:posterior_dynamics} are formally proven in Appendices~\ref{appendix:score_velocity}, \ref{appendix:proof_conditioning}, and \ref{appendix:proof_interpolant_likelihood}. Comprehensive experimental configurations—including architectures, metrics, baselines, and data generation—are detailed in Appendices~\ref{appendix:implementation} through~\ref{appendix:urban_case} (including specific configurations in \ref{subsec:ns_data}). The code used to produce the results presented in the paper can be found in \url{https://github.com/nmucke/scientific-stochastic-interpolants}.

\subsection*{Ethics statement}
The research presented here introduces a unified framework for Bayesian data assimilation using flow-based generative models. The methodology is evaluated exclusively on physical systems, specifically linear-Gaussian dynamics, two-dimensional Navier-Stokes equations, and urban airflow simulations. This work does not involve human subjects, sensitive personal datasets, or applications that risk generating harmful insights or propagating societal bias. The foundational contributions strictly advance physical forecasting methodologies and present no foreseeable negative societal impacts.   

\subsection*{AI use statement}
In this work, we used generative AI tools for formulating mathematical claims, providing critical ingredients for proving mathematical claims, assisting in the writing of proofs, proposing or refining hypotheses, designing or providing feedback on research methodology or experiments, implementing methods, supporting qualitative and thematic data analysis, and interpreting results. We have not used generative AI for generating synthetic data sets, developing theoretical models, or cleaning datasets; translation is not applicable to this work.

Additionally, we used AI for creating scientific figures, suggesting experimental parameters, writing/editing code, creating artifacts, drafting text, analyzing literature, searching for information, editing for readability, identifying literature, and formatting references.

We have reviewed all AI-assisted work. We take responsibility for the final content of this work, including text, claims, or artifacts produced with the aid of generative AI.

\bibliography{references}
\bibliographystyle{iclr2027_conference}

\appendix

\section{Velocity--score duality and the marginal-preserving SDE family}
\label{appendix:score_velocity}

This appendix establishes the structural facts on which the model-agnostic method rests: for an interpolant path the velocity and the score determine each other in closed form (Proposition~\ref{prop:score_velocity}), the probability-flow ODE belongs to a one-parameter family of SDEs sharing the same marginals (Proposition~\ref{prop:marginal_preservation}), and both velocity and score can be read off the trained network for each model class (Section~\ref{appendix:score_reading}). The first two are known~\citep{albergo_stochastic_2023, song_score-based_2021, ma_sit_2024}. We state and prove them in the notation of Section~\ref{subsec:flow_models} for completeness.

Throughout, consider the interpolant path of Eq.~\eqref{eq:general_interpolant},
\begin{align}\label{eq:app_path}
    \bx_\tau = \alpha_\tau\bx_0 + \beta_\tau\bx_1 + \sigpath_\tau\latentnoise,
    \qquad \latentnoise\sim\mathcal{N}(0,\bI)\ \text{independent of}\ \bx_1,
\end{align}
with $\bx_0$ deterministic, $\E\|\bx_1\|^2<\infty$, $\beta_\tau>0$ for $\tau\in(0,1]$, and schedules continuously differentiable with $\sigpath_\tau>0$ on the open interval $(0,1)$. All identities below are stated on $(0,1)$. Whenever the score is recovered from the velocity, we additionally require $\dot\beta_\tau\sigpath_\tau-\dot\sigpath_\tau\beta_\tau\neq0$. (For stochastic interpolants $\sigpath_\tau=\gamma_\tau\sqrt\tau$ vanishes and is not differentiable at $\tau=0$. The endpoints are covered by the initial condition and the terminal constraint $\bx_{\tau=1}=\bx_1$, consistent with the hypotheses of Theorem~\ref{theorem:unified_posterior}.) Let $p_\tau$ denote the marginal density of $\bx_\tau$ and $\genscore_\tau=\nabla_{\bx}\log p_\tau$ its score.

\begin{lemma}[Tweedie identities for the interpolant path]\label{lem:tweedie}
Writing $\bm m_\tau := \alpha_\tau\bx_0 + \beta_\tau\bx_1$, so that $\bx_\tau = \bm m_\tau + \sigpath_\tau\latentnoise$ with $\sigpath_\tau\latentnoise$ Gaussian with covariance $\sigpath_\tau^2\bI$ independent of $\bm m_\tau$,
\begin{align}
    \E[\bm m_\tau \mid \bx_\tau] &= \bx_\tau + \sigpath_\tau^2\,\genscore_\tau(\bx_\tau), \label{eq:tweedie_mean}\\
    \E[\latentnoise \mid \bx_\tau] &= -\sigpath_\tau\,\genscore_\tau(\bx_\tau), \label{eq:tweedie_eps}\\
    \E[\bx_1 \mid \bx_\tau] &= \tfrac{1}{\beta_\tau}\!\left(\bx_\tau + \sigpath_\tau^2\genscore_\tau(\bx_\tau) - \alpha_\tau\bx_0\right), \label{eq:tweedie_x1}\\
    \cov(\latentnoise \mid \bx_\tau) &= \bI + \sigpath_\tau^2\,\nabla_{\bx}\genscore_\tau(\bx_\tau). \label{eq:tweedie_cov}
\end{align}
\end{lemma}
\begin{proof}
Equation~\eqref{eq:tweedie_mean} is Tweedie's formula for the additive-Gaussian-noise model $\bx_\tau = \bm m_\tau + \sigpath_\tau\latentnoise$. Since $\latentnoise=(\bx_\tau-\bm m_\tau)/\sigpath_\tau$, \eqref{eq:tweedie_eps} follows, and $\bx_1=(\bm m_\tau-\alpha_\tau\bx_0)/\beta_\tau$ gives \eqref{eq:tweedie_x1}. For \eqref{eq:tweedie_cov}, differentiate $\E[\bm m_\tau\mid\bx_\tau=\bx]$ in $\bx$: on the one hand, by \eqref{eq:tweedie_mean}, $\nabla_{\bx}\E[\bm m_\tau\mid\bx] = \bI + \sigpath_\tau^2\nabla_{\bx}\genscore_\tau$. On the other, differentiating the conditional Gaussian mean directly gives the standard second-order Tweedie identity $\nabla_{\bx}\E[\bm m_\tau\mid\bx] = \sigpath_\tau^{-2}\cov(\bm m_\tau\mid\bx)$. Hence $\cov(\bm m_\tau\mid\bx)=\sigpath_\tau^2\bI+\sigpath_\tau^4\nabla_{\bx}\genscore_\tau$, and $\cov(\latentnoise\mid\bx)=\sigpath_\tau^{-2}\cov(\bm m_\tau\mid\bx)$ (as $\latentnoise=(\bx-\bm m_\tau)/\sigpath_\tau$ with $\bx$ fixed) yields \eqref{eq:tweedie_cov}.
\end{proof}

\begin{proposition}[Velocity--score identity]\label{prop:score_velocity}
The marginal velocity and the score of the path \eqref{eq:app_path} satisfy the exact linear relation
\begin{align}\label{eq:general_velocity_of_score}
    \fmvel_\tau(\bx)
    = \frac{\dot\beta_\tau}{\beta_\tau}\bx
    + \left(\dot\alpha_\tau - \frac{\dot\beta_\tau\alpha_\tau}{\beta_\tau}\right)\bx_0
    + \sigpath_\tau\!\left(\frac{\dot\beta_\tau}{\beta_\tau}\sigpath_\tau - \dot\sigpath_\tau\right)\genscore_\tau(\bx),
\end{align}
where the score coefficient is the velocity--score coefficient $\vscoef_\tau=\sigpath_\tau(\tfrac{\dot\beta_\tau}{\beta_\tau}\sigpath_\tau-\dot\sigpath_\tau)$ of Eq.~\eqref{eq:general_score_velocity}. When $\dot\beta_\tau\sigpath_\tau-\dot\sigpath_\tau\beta_\tau\neq0$, this is equivalently
\begin{align}\label{eq:app_score_of_velocity}
    \genscore_\tau(\bx)
    = \frac{\beta_\tau\,\fmvel_\tau(\bx) - \dot\beta_\tau\,\bx - (\dot\alpha_\tau\beta_\tau - \dot\beta_\tau\alpha_\tau)\,\bx_0}
           {\sigpath_\tau\,(\dot\beta_\tau\sigpath_\tau - \dot\sigpath_\tau\beta_\tau)} .
\end{align}
\end{proposition}
\begin{proof}
By definition $\fmvel_\tau(\bx)=\E[\dot\alpha_\tau\bx_0 + \dot\beta_\tau\bx_1 + \dot\sigpath_\tau\latentnoise\mid\bx_\tau=\bx]$. Since $\bx_0$ is deterministic, substitute \eqref{eq:tweedie_x1} and \eqref{eq:tweedie_eps}:
\begin{align}
    \fmvel_\tau(\bx)
    &= \dot\alpha_\tau\bx_0
    + \frac{\dot\beta_\tau}{\beta_\tau}\!\left(\bx + \sigpath_\tau^2\genscore_\tau - \alpha_\tau\bx_0\right)
    - \dot\sigpath_\tau\sigpath_\tau\genscore_\tau \nonumber\\
    &= \frac{\dot\beta_\tau}{\beta_\tau}\bx
    + \left(\dot\alpha_\tau - \frac{\dot\beta_\tau\alpha_\tau}{\beta_\tau}\right)\bx_0
    + \left(\frac{\dot\beta_\tau}{\beta_\tau}\sigpath_\tau^2 - \dot\sigpath_\tau\sigpath_\tau\right)\genscore_\tau,
\end{align}
which is \eqref{eq:general_velocity_of_score}. Multiplying the isolated score term by $\beta_\tau$ and rearranging yields \eqref{eq:app_score_of_velocity}.
\end{proof}

\begin{proposition}[Marginal-preserving SDE family]\label{prop:marginal_preservation}
Let $p_\tau$ be the marginals of the path \eqref{eq:app_path} with velocity $\fmvel_\tau$ and score $\genscore_\tau$, and suppose $p\in C^{2}((0,1)\times\R^{\statedim})$. Let $\gdiff_\tau\ge0$ be measurable with $\gdiff\in L^2([0,1])$. Assume that the SDE
\begin{align}\label{eq:app_sde}
    \rd\bX_\tau = \left[\fmvel_\tau(\bX_\tau) + \tfrac12\gdiff_\tau^2\genscore_\tau(\bX_\tau)\right]\rd\tau + \gdiff_\tau\,\rd\bW_\tau,
    \qquad \bX_0\sim p_0,
\end{align}
and its Fokker--Planck equation are well posed with unique solutions for the initial law $p_0$ (allowing the point mass $p_0=\delta_{\bx_0}$ used by SI, which has no density). Then the SDE has marginal law $p_\tau$ for all $\tau\in[0,1]$. Setting $\gdiff_\tau=0$ gives the probability-flow ODE whenever that ODE is well posed.
\end{proposition}
\begin{proof}
Since $\genscore_\tau=\nabla\log p_\tau$, we have
$p_\tau\genscore_\tau=\nabla p_\tau$. Hence the score drift and diffusion cancel in the Fokker--Planck equation:
\begin{align*}
    -\nabla\!\cdot\!\left(
        \left[\fmvel_\tau+\tfrac12\gdiff_\tau^2\genscore_\tau\right]p_\tau
    \right)
    +\tfrac12\gdiff_\tau^2\Delta p_\tau
    =-\nabla\!\cdot\!(\fmvel_\tau p_\tau)
    =\partial_\tau p_\tau,
\end{align*}
where the last equality is the interpolant's continuity equation.
Thus $p_\tau$ solves the Fokker--Planck equation with initial law $p_0$
(interpreted weakly when $p_0$ is a point mass).
Uniqueness gives the claimed marginals. Taking $\gdiff_\tau=0$ yields
the probability-flow ODE.
\end{proof}

\subsection{Reading the velocity and score from the trained network}
\label{appendix:score_reading}

The two model classes train different objects. Proposition~\ref{prop:score_velocity} converts either into the pair $(\fmvel_\tau,\genscore_\tau)$ that the samplers use.

\paragraph{Flow matching.}
The trained object is the velocity $\fmvel_\theta$ itself, and the path has $\alpha_\tau=0$ with source schedule $\sigpath_\tau$. Specialising \eqref{eq:app_score_of_velocity},
\begin{align}\label{eq:fm_score}
    \genscore_\tau(\bx) = \frac{\beta_\tau\,\fmvel_\theta(\bx,\state^{n-1},\tau) - \dot\beta_\tau\,\bx}{\sigpath_\tau\,(\dot\beta_\tau\sigpath_\tau - \dot\sigpath_\tau\beta_\tau)},
\end{align}
which for rectified flow ($\sigpath_\tau=1-\tau$, $\beta_\tau=\tau$) reduces to $\genscore_\tau(\bx) = (\tau\,\fmvel_\theta - \bx)/(1-\tau)$. Note that \eqref{eq:fm_score} degenerates as $\tau\to1$ (where $\sigpath_\tau\to0$), which is why the DM-SDE lift uses an endpoint-vanishing diffusion $\gdiff_\tau$.

\paragraph{Stochastic interpolants.}
The trained object is the SDE drift $\drift_\tau = \E[\dot\alpha_\tau\bx_0 + \dot\beta_\tau\bx_1 + \dot\gamma_\tau\bW_\tau\mid\bX_\tau=\bx]$ of Eq.~\eqref{eq:si_sde}, with $\sigpath_\tau=\gamma_\tau\sqrt\tau$. It decomposes as $\drift_\tau=\fmvel_\tau+\tfrac12\gamma_\tau^2\genscore_\tau$~\citep{chen_probabilistic_2024}, which follows from Lemma~\ref{lem:tweedie}: since $\gamma_\tau\bW_\tau:=\sigpath_\tau\latentnoise$ at fixed $\tau$, Eq.~\eqref{eq:tweedie_eps} gives
\begin{align*}
    \E[\dot\gamma_\tau\bW_\tau\mid\bX_\tau=\bx]
    = \tfrac{\dot\gamma_\tau}{\gamma_\tau}\,\E[\gamma_\tau\bW_\tau\mid\bX_\tau=\bx]
    = -\dot\gamma_\tau\gamma_\tau\tau\,\genscore_\tau(\bx),
\end{align*}
while the noise term of the drift carries $\dot\sigpath_\tau\,\E[\latentnoise\mid\bX_\tau=\bx] = -\dot\sigpath_\tau\sigpath_\tau\,\genscore_\tau(\bx)$ with $\dot\sigpath_\tau\sigpath_\tau = \tfrac12\tfrac{\rd}{\rd\tau}(\gamma_\tau^2\tau) = \dot\gamma_\tau\gamma_\tau\tau + \tfrac12\gamma_\tau^2$. The remaining terms coincide, so $\drift_\tau - \fmvel_\tau = \tfrac12\gamma_\tau^2\genscore_\tau$. This is also the reconciliation of the two noise parametrisations: the $\dot\gamma_\tau\bW_\tau$ regression target (the It\^o drift of $\rd(\gamma_\tau\bW_\tau)$ for a genuine Wiener process) and the pathwise derivative $\dot\sigpath_\tau\latentnoise$ of the equal-in-law source $\sigpath_\tau\latentnoise$ differ, upon conditioning, by exactly the score term that the SDE's diffusion restores. Adding $\tfrac12\gamma_\tau^2\genscore_\tau$ to \eqref{eq:general_velocity_of_score} and solving for the score,
\begin{align}\label{eq:SI_score}
    \genscore_\tau(\bx,\bx_0) = A_\tau\!\left[\beta_\tau\drift_\tau(\bx,\bx_0,\tau) - \bm c_\tau(\bx,\bx_0)\right],
\end{align}
with $A_\tau = \left[\tau\gamma_\tau(\dot\beta_\tau\gamma_\tau - \beta_\tau\dot\gamma_\tau)\right]^{-1}$ and $\bm c_\tau(\bx,\bx_0)=\dot\beta_\tau\bx + (\beta_\tau\dot\alpha_\tau - \dot\beta_\tau\alpha_\tau)\bx_0$ as stated by~\citet{chen_probabilistic_2024}. The velocity, when needed separately, is $\fmvel_\tau=\drift_\tau-\tfrac12\gamma_\tau^2\genscore_\tau$.

\section{Conditioning the interpolant path: proofs}
\label{appendix:proof_conditioning}

Throughout, $\bx_0$ is fixed and
$\obs=\obsoperator(\bx_1)+\obsnoise$, with $\obsnoise$ independent
of $(\bx_1,\latentnoise)$ given $\bx_0$.

\subsection{Posterior path}

\begin{proposition}[Conditioning identities]\label{prop:conditional_path}
The conditioned path retains the interpolant representation
\eqref{eq:general_interpolant}, with the endpoint drawn from the
posterior:
\begin{align}\label{eq:conditional_path}
\bx_\tau=\alpha_\tau\bx_0+\beta_\tau\bx_1+\sigpath_\tau\latentnoise,
    \qquad
    \bx_1\sim\conddist{\bx_1}{\obs,\bx_0},
    \quad \latentnoise\sim\mathcal{N}(0,\bI),
\end{align}
with independent $\bx_1$ and $\latentnoise$. Denoting its score
and velocity by $\genscore_\tau^\star$ and $\fmvel_\tau^\star$,
respectively, we have, for $\tau\in(0,1)$,
\begin{align}
    \genscore_\tau^\star
    &= \genscore_\tau+\nabla_{\bx}\log\Phi_\tau^{\obs},
    \label{eq:conditional_score}\\
    \fmvel_\tau^\star
    &= \fmvel_\tau+\vscoef_\tau\nabla_{\bx}\log\Phi_\tau^{\obs},
    \label{eq:conditional_velocity}
\end{align}
where arguments $(\bx,\bx_0)$ are suppressed,
$\Phi_\tau^{\obs}$ is defined in \eqref{eq:likelihood_tilt},
and $\vscoef_\tau$ in \eqref{eq:general_score_velocity}.
\end{proposition}

\begin{proof}
Since $\latentnoise$ is independent of $(\bx_1,\obs)$ given $\bx_0$,
conditioning changes only the law of $\bx_1$, giving
\eqref{eq:conditional_path}. By Bayes' rule,
\begin{align}\label{eq:posterior_marginal}
    p_\tau(\bx\mid\obs, \bx_0)
    =\frac{\Phi_\tau^{\obs}(\bx,\bx_0)\,p_\tau(\bx\mid\bx_0)}
    {\conddist{\obs}{\bx_0}}.
\end{align}
Taking $\nabla_{\bx}\log$ gives \eqref{eq:conditional_score}.
Applying the velocity--score identity
\eqref{eq:general_velocity_of_score} to both paths, whose schedules
and source coincide, then gives \eqref{eq:conditional_velocity}.
\end{proof}

\subsection{Proof of Theorem~\ref{theorem:unified_posterior}}

\begin{proof}
By Proposition~\ref{prop:conditional_path}, the posterior marginals
form an interpolant path. Under the theorem's well-posedness
assumptions, Proposition~\ref{prop:marginal_preservation} therefore
gives a sampler with these marginals:
\begin{align*}
    \rd\bXstar_\tau
    = \left[\fmvel_\tau^\star(\bXstar_\tau,\bx_0)
    + \tfrac12\gdiff_\tau^2\,\genscore_\tau^\star(\bXstar_\tau,\bx_0)\right]\rd\tau
    + \gdiff_\tau\,\rd\bW_\tau,
    \qquad \bXstar_0\sim p_0^{\obs}.
\end{align*}
Substituting \eqref{eq:conditional_score}--\eqref{eq:conditional_velocity}
into its drift and suppressing the common arguments gives
\begin{align*}
    \fmvel_\tau^\star+\tfrac12\gdiff_\tau^2\genscore_\tau^\star
    = \underbrace{\fmvel_\tau+\tfrac12\gdiff_\tau^2\genscore_\tau}_{\text{prior drift}}
    + \underbrace{\big(\vscoef_\tau+\tfrac12\gdiff_\tau^2\big)}_{\gweight_\tau}
      \nabla_{\bx}\log\Phi_\tau^{\obs},
\end{align*}
which is exactly \eqref{eq:unified_posterior_drift}. At $\tau=1$,
the interpolant equals $\bx_1$, so the sampler's terminal law is
$\conddist{\bx_1}{\obs,\bx_0}$.

Finally, the initial law is unchanged by conditioning. Since
$\beta_0=0$, the source
$\bx_{\tau=0}=\alpha_0\bx_0+\sigpath_0\latentnoise$ is independent
of $\obs$ given $\bx_0$. Thus
$\Phi_0^{\obs}=\conddist{\obs}{\bx_0}$ on the source support and
$p_0^{\obs}=p_0$. The sampler starts from $\mathcal{N}(0,\bI)$
for FM/DM and from $\delta_{\bx_0}$ for SI, including when the
source has no density.
\end{proof}

\section{Proof of Lemma \ref{lem:interpolated_likelihood} and Lemma~\ref{lem:source_moments}} \label{appendix:proof_interpolant_likelihood}

\begin{proof}[Proof of Lemma~\ref{lem:interpolated_likelihood}]
Recall the general state interpolant and the observation interpolant,
\begin{align}
  \bx_\tau &= \alpha_\tau \bx_0 + \beta_\tau \bx_1 + \src_\tau,
     \qquad \src_\tau = \sigpath_\tau\latentnoise, \label{eq:pf_state_interp}\\
  \interpolantobs_\tau &= \alpha_\tau \obsmatrix\bx_0 + \beta_\tau \obs, \label{eq:pf_obs_interp}
\end{align}
where $\obs = \obsmatrix\bx_1 + \obsnoise$ with $\obsnoise \sim \mathcal{N}(0,\obscov)$ independent of the source driver $\latentnoise$. Substituting $\obs = \obsmatrix\bx_1 + \obsnoise$ into \eqref{eq:pf_obs_interp},
\begin{equation}
  \interpolantobs_\tau
  = \obsmatrix(\alpha_\tau\bx_0 + \beta_\tau\bx_1) + \beta_\tau\obsnoise .
\end{equation}
From~\eqref{eq:pf_state_interp}, $\alpha_\tau\bx_0 + \beta_\tau\bx_1 = \bx_\tau - \src_\tau$, giving
\begin{equation}\label{eq:pf_ybar_decomp}
  \interpolantobs_\tau
  = \obsmatrix\bx_\tau
    + \underbrace{\beta_\tau\obsnoise - \obsmatrix\src_\tau}_{=\,\interpolantobsnoise_\tau},
\end{equation}
which is Eq.~\eqref{eq:ybar_decomp}. Since $\obsnoise$ is independent of $\bx_1$ and of the source driver, $\obsnoise \perp (\src_\tau, \bx_\tau) \mid \bx_0$, so
\begin{equation}
  \E[\obsnoise \mid \bx_\tau, \bx_0] = 0,
  \qquad
  \cov(\obsnoise \mid \bx_\tau, \bx_0) = \obscov .
\end{equation}
Taking conditional expectation of \eqref{eq:pf_ybar_decomp},
\begin{align}
  \E[\interpolantobs_\tau \mid \bx_\tau, \bx_0]
  = \obsmatrix\bx_\tau + \beta_\tau\,\E[\obsnoise \mid \bx_\tau, \bx_0]
    - \obsmatrix\,\E[\src_\tau \mid \bx_\tau, \bx_0]
  = \obsmatrix\bx_\tau - \obsmatrix\,\srcmean_\tau,
\end{align}
which is the mean in \eqref{eq:likelihood_moments}. Because the cross terms between $\obsnoise$ and $\src_\tau$ vanish upon conditioning,
\begin{align}
  \cov(\interpolantobs_\tau \mid \bx_\tau, \bx_0)
  = \beta_\tau^2\,\cov(\obsnoise \mid \bx_\tau, \bx_0)
    + \obsmatrix\,\cov(\src_\tau \mid \bx_\tau, \bx_0)\,\obsmatrix^\top
  = \beta_\tau^2\obscov + \obsmatrix\,\srccov_\tau\,\obsmatrix^\top,
\end{align}
which is the covariance in \eqref{eq:likelihood_moments}.
\end{proof}

\begin{proof}[Proof of Lemma~\ref{lem:source_moments}]
The source process is $\src_\tau = \sigpath_\tau\latentnoise$, so $\srcmean_\tau = \sigpath_\tau\,\E[\latentnoise\mid\bx_\tau,\bx_0]$ and $\srccov_\tau = \sigpath_\tau^2\,\cov(\latentnoise\mid\bx_\tau,\bx_0)$. The Tweedie identities of Lemma~\ref{lem:tweedie} give $\E[\latentnoise\mid\bx_\tau]=-\sigpath_\tau\genscore_\tau$ \eqref{eq:tweedie_eps} and $\cov(\latentnoise\mid\bx_\tau)=\bI+\sigpath_\tau^2\nabla_{\bx}\genscore_\tau$ \eqref{eq:tweedie_cov}, hence
\begin{align}
    \srcmean_\tau = -\sigpath_\tau^2\,\genscore_\tau(\bx,\bx_0),
    \qquad
    \srccov_\tau = \sigpath_\tau^2\,\bI + \sigpath_\tau^4\,\nabla_{\bx}\genscore_\tau(\bx,\bx_0),
\end{align}
which is Eq.~\eqref{eq:source_moments}.
\end{proof}

\section{Implementation details}
\label{appendix:implementation}

\paragraph{Interpolant schedules.}
SI priors use the quadratic schedule $\alpha_\tau=1-\tau$, $\beta_\tau=\tau^2$, $\gamma_\tau=1-\tau$ (so $\sigpath_\tau=(1-\tau)\sqrt\tau$), and the FM priors use linear schedules with only a noise source term, $\alpha_\tau=0$, $\beta_\tau=\tau$, $\sigma_\tau=1-\tau$.

\paragraph{Pseudo-time integration.}
Sampling uses $M$ uniform pseudo-time steps: Euler--Maruyama for the SDE members and forward Euler for the guided ODE.

\paragraph{Architecture.}
The field-case priors (SI drift and FM velocity networks alike) are conditional U-Nets: four resolution levels with two ConvNeXt-style blocks per level (inner channel multiplier 2), circular padding, and multi-head spatial attention (four heads, rotary positional embedding) in the bottleneck. Pseudo-time enters through a learned scalar embedding, and the five most recent physical states are supplied as additional input channels to the first convolution, realising the conditioning on $\state^{n-1}$ (and a short history). The initial conditioning history consists of five consecutive states from the trajectory supplying the initial condition. Thereafter, each ensemble member updates its own history with its generated states. Channel widths per level: Navier--Stokes $(8,16,32,64)$ for the SI prior and $(16,32,64,128)$ for the FM prior, one field channel (vorticity). Urban $(32,64,128,256)$, four field channels (three velocity components and temperature), with solid cells excluded via the building mask. Hyperparameters were chosen based on hyperparameter tuning.

\paragraph{Training.}
All priors are trained by simulation-free regression of their target (SI drift or FM velocity) under an MSE loss, with per-channel standardised fields: AdamW (learning rate $10^{-4}$, weight decay $10^{-6}$) with cosine annealing to $10^{-5}$, batch size 32, up to $10^{3}$ epochs with early stopping, dropout $0.1$, adaptive EMA-based gradient clipping, and mixed precision. Navier--Stokes uses 160 training trajectories and 20 validation trajectories for the early stopping. The urban case uses 150 LES runs for training and 20 runs for validation. The same trained FM checkpoint per case serves every FM and DM sampler and the same trained SI checkpoint serve every SI sampler. No prior is ever retrained or fine-tuned for assimilation.

\section{Evaluation metrics}
\label{appendix:metrics}

All metrics are computed per physical step from the posterior ensemble $\{\bx^{(e)}\}_{e=1}^{E}$ against the truth $\bx^{*}$ (or the exact posterior in the analytical case) and then averaged over steps and trajectories. Spatial averages run over the kept grid points $i\in\mathcal{I}$ (in the urban case, solid cells are excluded); $\bar\bx = \frac{1}{E}\sum_{e}\bx^{(e)}$ denotes the ensemble mean. The urban per-variable metrics (velocity, temperature) apply the same formulas with $\mathcal{I}$ restricted to the channels of that variable.

\paragraph{Ensemble-mean RMSE.} The RMSE of the ensemble mean,
\begin{align}
    \mathrm{RMSE} = \Big(\tfrac{1}{|\mathcal{I}|}\textstyle\sum_{i\in\mathcal{I}} \big(\bar x_i - x^{*}_i\big)^2\Big)^{1/2},
\end{align}
not the member-averaged RMSE, which would additionally count ensemble spread.

\paragraph{CRPS.} Per grid point, the unbiased ensemble estimator of the continuous ranked probability score~\citep{gneiting_strictly_2007},
\begin{align}
    \mathrm{CRPS}_i = \frac{1}{E}\sum_{e=1}^{E}\big|x^{(e)}_i - x^{*}_i\big|
    - \frac{1}{2E(E-1)}\sum_{e\neq e'}\big|x^{(e)}_i - x^{(e')}_i\big|,
\end{align}
averaged over $\mathcal{I}$; the $1/(E(E-1))$ normalisation makes the spread term unbiased for $\E|X-X'|$.

\paragraph{Spread--skill.} The ratio of corrected mean ensemble spread to skill,
\begin{align}
    \mathrm{ratio} = \sqrt{\tfrac{E+1}{E}}\,\Big(\tfrac{1}{|\mathcal{I}|}\textstyle\sum_{i\in\mathcal{I}} s_i\Big) \Big/\, \mathrm{RMSE},
\end{align}
with $s_i$ the ensemble standard deviation at point $i$ and $\sqrt{(E+1)/E}$ the finite-ensemble correction. We report $|1-\mathrm{ratio}|$, so $0$ is perfectly calibrated.

\paragraph{KL from the exact posterior (analytical case).} With sample moments $(\mu_p,\Sigma_p)$ of the method ensemble and exact posterior moments $(\mu_q,\Sigma_q)$, we report the Gaussian KL from the reference to the fitted ensemble distribution:
\begin{align}
    D_{\mathrm{KL}}(q\|p)
    = \tfrac{1}{2}\big(
    \operatorname{tr}(\Sigma_p^{-1}\Sigma_q)
    + (\mu_p-\mu_q)^\top\Sigma_p^{-1}(\mu_p-\mu_q)
    - d + \log\det\Sigma_p - \log\det\Sigma_q
    \big),
\end{align}
where $d$ is the state dimension. Because $\mu_p$ and $\Sigma_p$ are estimated from a finite ensemble, this reported KL is generally nonzero even for independent samples from the exact posterior. Increasing the number of integration steps alone does not remove this sampling error.

\paragraph{Sliced $W_1$.}
The 1-Wasserstein distance averaged over random one-dimensional projections: for unit directions $\theta_j$, $j=1,\dots,128$, the projected distance $W_1(\theta_j)$ is computed by integrating the absolute difference between the empirical quantile functions of the two projected sample sets, and we report $\tfrac{1}{128}\sum_{j=1}^{128} W_1(\theta_j)$.

\paragraph{Cost.} Wall-clock seconds per assimilation step, at matched ensemble size on the same hardware.

\paragraph{Enstrophy spectrum (Navier--Stokes).} The spectral diagnostic of the reconstructions, the radially averaged enstrophy spectrum $Z(k)$, is defined where it is used, in Eq.~\eqref{eq:ns_enstrophy_spectrum} of Appendix~\ref{subsec:ns_spectra}.

\section{Method descriptions and comparison to alternatives}\label{appendix:methods}

Our method shares building blocks with the baselines below but assembles them differently. For each method we state how the state is evolved, how the likelihood term is approximated, the per-step network cost, and which hyperparameters were tuned; quantitative comparisons are in Section~\ref{sec:results}. Tuned hyperparameters were selected by RMSE sweeps per observation scenario (and, where indicated, per step count $M$) on a held-out trajectory not used for evaluation; Table~\ref{tab:method_hparams} collects all sampler hyperparameters. Throughout, $\hat{\state}^1_\tau = \E[\bx_1\mid\bx_\tau,\state^0]$ denotes the one-step denoiser.

\paragraph{FlowDAS.}
FlowDAS~\citep{chen_flowdas_2025}, commits to the SI-SDE and estimates the likelihood score by Monte Carlo: the gradient of a softmax-weighted squared residual through $J$ one-step predictions $\hat{\state}^{1,j}_\tau$, entering the drift with a tuned constant step size $\zeta$ in place of our analytic weight $\gweight_\tau$. We instead read the score off the closed-form observation interpolant, removing sampling variance and per-step trajectory backprop; the trade-off is our Gaussian surrogate, whose bias is controlled by the covariance choice rather than the number of samples. \emph{Cost:} two forwards (prior drift and denoiser endpoint) and one backward per step, independent of $J$. \emph{Key limitation:} $\zeta$ is scenario-dependent and may require tuning, and there is no inflated source covariance. In our implementation, stability can also depend on the solver step count $M$, and overly large $\zeta$ can lead to divergence. \emph{Tuned:} the guidance step size $\zeta$, per scenario and $M$; the number of one-step samples is fixed at $J=25$.

\paragraph{Guided FM (FIG).}
FIG~\citep{yan_fig_2024} guides the deterministic FM probability-flow ODE for static linear inverse problems: after each prior-velocity step, $k$ inner corrector steps push the state along $\obsmatrix^\top$ times the \emph{normalised} observation residual with magnitude $c\,(1-\tau)/\tau$, using a measurement interpolant equivalent to ours. \emph{Cost:} one forward per step, no backward (the corrector differentiates only through the linear $\obsmatrix$) -- the cheapest guided sampler in the comparison. \emph{Key limitation:} treats $\bx_\tau\approx\bx_1$, uses no covariance, and the corrector is not derived from a probability path, so the guidance magnitude $c(1-\tau)/\tau$ is not data-adaptive. \emph{Tuned:} the corrector scale $c$, per scenario and $M$; the number of inner corrector steps $k$ is set by the observation operator (sweep-selected: $k=1$ for super-resolution, $k=3$ for sparse sensors), and the measurement-interpolant noise is fixed to zero as in the original work.

\paragraph{SDA.}
SDA~\citep{rozet_score-based_2023} applies Tweedie-based guidance on a diffusion prior: the gradient of a Gaussian likelihood at the Tweedie denoiser, taken through the network, with a $\Pi$GDM-style inflated isotropic covariance scaled by a tunable factor $\gamma_{\mathrm{sda}}$. Its original form composes locally learned segment scores into a joint trajectory score and conditions all states at once. We use its single-window form and are autoregressive, keeping per-step cost and prior dimension fixed as the horizon grows (but not propagating cross-window correlations as a joint smoother would). \emph{Cost:} two forwards and one backward per step. \emph{Tuned:} the covariance inflation $\gamma_{\mathrm{sda}}$, per scenario and $M$. \emph{Key limitation:} SDA is tied to the native reverse diffusion. Its bias therefore stems entirely from the Gaussian--Tweedie likelihood surrogate, whose isotropic covariance carries no model source information. Because the guidance weight is inherited from the fixed diffusion rather than derived from a probability path, it does not extend to samplers with modified or vanishing diffusion ($\gweight_\tau\neq\gdiff_\tau^2$), unlike Theorem~\ref{theorem:unified_posterior}.

\paragraph{D-Flow SGLD.}
Original D-Flow~\citep{ben-hamu_d-flow_2024} optimizes the FM source $\bx_0$ by gradient descent through the full FM-ODE rollout, with prior regularization supplied implicitly by the learned flow. D-Flow SGLD~\citep{parikh_d-flow_2026}, used here, instead runs a preconditioned SGLD chain over $\bx_0\sim\mathcal{N}(0,I)$ through the rollout (6-step RK2, backpropagated with gradient checkpointing), with energy given by the squared data misfit plus an explicit $\ell_2$ source penalty $\lambda\|\bx_0\|^2$; zero Langevin noise recovers MAP-style source optimization. \emph{Cost:} the entire rollout repeats at every SGLD step ($12$ forwards and $12$ backwards per step at our settings) -- the only baseline backpropagating through a full multi-step generative rollout, and by far the most expensive. \emph{Tuned:} the SGLD step size $\eta$ and the source penalty $\lambda$, per case (in our sweeps $\eta$ is the dominant knob); the Langevin noise scale and the preconditioner constants follow the original paper.

\paragraph{SURGE.}
SURGE~\citep{wei_surge_2026} is a sequential Monte-Carlo wrapper: a guided reverse-SDE proposal with Girsanov-corrected importance weights that telescope to the exact terminal likelihood, making the target asymptotically exact as the number of particles $N\to\infty$. The guidance controls proposal efficiency and therefore finite-$N$ degeneracy. In our implementation, particles resample when the effective sample size drops below $N/2$. We report it wrapping the FlowDAS and SDA guidance. \emph{Measured implementation cost:} one additional denoiser forward per step relative to the wrapped method, with no additional backward pass. \emph{Tuned:} nothing SURGE-specific -- the resampling threshold is fixed at $\mathrm{ESS}<N/2$, and the wrapped guidance keeps the hyperparameters tuned for its standalone counterpart ($\zeta$, $\gamma_{\mathrm{sda}}$). \emph{Key limitation:} particle degeneracy in high dimension.

\paragraph{Diffusion prior for DM-based methods.}
SDA, SURGE, and DM-SDE require a diffusion prior; we construct it from the trained FM velocity via the velocity--score identity (Appendix~\ref{appendix:score_velocity}), so all generative methods use the same network.

\begin{table}[ht]
    \centering
    \small
    \caption{Sampler hyperparameters for all methods. \emph{Tuned} entries were selected by RMSE sweeps on a held-out trajectory, per observation scenario and -- where a range is given -- per step count $M$; ranges span the values used across the scenario/$M$ grid. \emph{Fixed} entries follow the original works. In the linear--Gaussian case our samplers use the exact covariance without damping ($\lambda=1$) and the baselines reuse the Navier--Stokes tuning (FIG: $k=1$, $c=1$). The full per-scenario, per-$M$ matrices are provided in the configuration files released with the code.}
    \label{tab:method_hparams}
    \resizebox{\ifdim\width>\textwidth\textwidth\else\width\fi}{!}{%
    \begin{tabular}{llllll}
        \toprule
        Method & Hyperparameter & Role & Selection & Navier--Stokes & Urban \\
        \midrule
        Ours (all samplers) & $\lambda$ & Jacobian damping (shared cov.) & tuned & $0.9$ (SR) / $0.95$ (sparse) & $0.95$ \\
                            & $k$ & Jacobian refresh cadence & fixed (cost dial) & $10$ & $10$ \\
        \midrule
        FlowDAS & $\zeta$ & guidance step size & tuned & $2\times10^{-3}$--$0.15$ & $5\times10^{-4}$ \\
                & $J$ & Monte-Carlo $\bx_1$ samples & fixed & $25$ & $25$ \\
        \midrule
        Guided FM (FIG) & $c$ & corrector step-size scale & tuned & $5$--$20$ & $160$ \\
                        & $k$ & inner corrector steps & operator type & $1$ (SR) / $3$ (sparse) & $3$ \\
                        & $w$ & measurement-interpolant noise & fixed & $0$ & $0$ \\
        \midrule
        SDA & $\gamma_{\mathrm{sda}}$ & likelihood covariance inflation & tuned & $10^{-4}$--$10^{-2}$ & $5\times10^{-4}$ \\
        \midrule
        D-Flow SGLD & $\eta$ & SGLD step size & tuned & $5\times10^{-3}$ & $5\times10^{-3}$ \\
                    & $\lambda$ & source penalty & tuned & $10^{-4}$ & $10^{-3}$ \\
                    & $s$ & Langevin noise scale & fixed & $10^{-3}$ & $10^{-3}$ \\
                    & $\omega$, $\delta$ & pSGLD preconditioner & fixed & $0.99$, $10^{-3}$ & $0.99$, $10^{-3}$ \\
                    & $K$ & SGLD steps (6-step RK2 rollout) & $=M$ & $M$ & $M$ \\
        \midrule
        SURGE & ESS threshold & resampling trigger & fixed & $N/2$ & $N/2$ \\
              & $\zeta$, $\gamma_{\mathrm{sda}}$ & wrapped guidance & carried over & as above & as above \\
        \midrule
        EnKF & inflation / localisation & analysis regularisation & none & $1.0$ / none & -- \\
        Particle filter & resampling & scheme & fixed & systematic & -- \\
        \bottomrule
    \end{tabular}}
\end{table}

\section{Analytical test case: setup and additional results}
\label{appendix:analytical_case}

\subsection{Setup}
\label{appendix:analytical_setup}
The analytical test case is the linear-Gaussian system with $\statedim=\obsdim=2$
\begin{align}
    \conddist{\state^1}{\state^0} = \mathcal{N}(\state^0, \bI),
    \qquad
    \conddist{\obs^1}{\state^1} = \mathcal{N}(\obsmatrix\state^1, \obscov),
    \qquad \obsmatrix = \bI,\ \obscov = \bI,
\end{align}
i.e.\ Eq.~\eqref{eq:data_assimilation_problem} with identity forward and observation operators and unit Gaussian noise.

For a Gaussian conditional target $\conddist{\bx_1}{\bx_0}=\mathcal{N}(\mean_1(\bx_0), \covmatrix_1)$, the SI drift is available in closed form (Proposition B.9 of~\citealp{chen_probabilistic_2024}). Specialised to $\mean_1(\bx_0)=\bx_0$ and $\covmatrix_1=\bI$, it reads
\begin{align}
    \drift(\bx, \bx_0, \tau) &= \dot{\alpha}_\tau \bx_0 + \dot{\beta}_\tau\mean_1(\bx_0)  + (\beta_\tau \dot{\beta}_\tau + \tau \gamma_\tau \dot{\gamma}_\tau)\, \bar{\covmatrix}_\tau^{-1}(\bx - \bar{\mean}_\tau(\bx_0)),
\end{align}
where
\begin{align}
    \bar{\mean}_\tau(\bx_0) = \alpha_\tau \bx_0 + \beta_\tau \mean_1(\bx_0), \quad \bar{\covmatrix}_\tau = (\beta_\tau^2 + \tau \gamma_\tau^2)\,\bI,
\end{align}
are the mean and covariance of the interpolant $\bx_\tau$ conditioned on $\bx_0$. Note that the second drift term involves the prior mean $\mean_1(\bx_0)=\bx_0$, not $\bar\mean_\tau(\bx_0)$; the two coincide only for schedules with $\alpha_\tau+\beta_\tau=1$.

The flow matching velocity is available in closed form for the same target. For linear schedules ($\sigpath_\tau = 1-\tau$, $\beta_\tau=\tau$) and $\mean_1(\bx_0)=\bx_0$ this reduces to
\begin{align}
    \fmvel(\bx, \bx_0, \tau) &= \bx_0 + \frac{2\tau-1}{\tau^2 + (1-\tau)^2}\,(\bx - \tau\bx_0).
\end{align}
The same velocity drives both the FM-ODE and, via the score identity \eqref{eq:fm_score}, the DM-SDE sampler with the endpoint-vanishing diffusion $\gdiff_\tau = \sqrt{\sigpath_\tau\beta_\tau}$.

The true posterior is
\begin{align}
    \conddist{\state^1}{\obs^1, \state^0} = \mathcal{N}\left(\state^0 + K (\obs^1 - \obsmatrix\state^0),\, \bI - K\obsmatrix \right),
\end{align}
where
\begin{align}
    K = \obsmatrix^\top (\obsmatrix\obsmatrix^\top + \obscov)^{-1} = \tfrac12\,\bI .
\end{align}

\subsection{Additional results}
\label{appendix:analytical_results}

Figure~\ref{fig:analytical_panels} visualises the case in two dimensions: the prior conditional, likelihood, exact posterior, and the posterior sampled with the closed-form drift. Table~\ref{tab:analytical_accuracy} reports, for every method, the KL divergence and the sliced-$W_1$ distance to the exact posterior at $M=50$ and $M=500$, alongside the wall-clock cost of one assimilation.

\begin{figure*}[t]
    \centering
    \begin{subfigure}{0.24\linewidth}\centering
        \includegraphics[width=\linewidth]{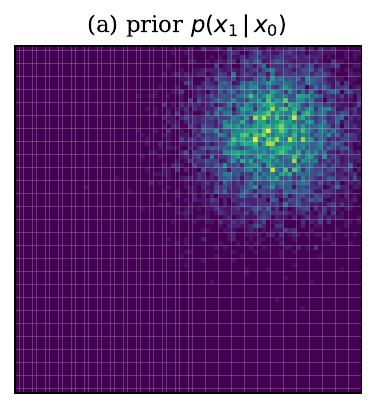}
        \caption{Prior conditional.}\label{fig:an_prior}
    \end{subfigure}\hfill
    \begin{subfigure}{0.24\linewidth}\centering
        \includegraphics[width=\linewidth]{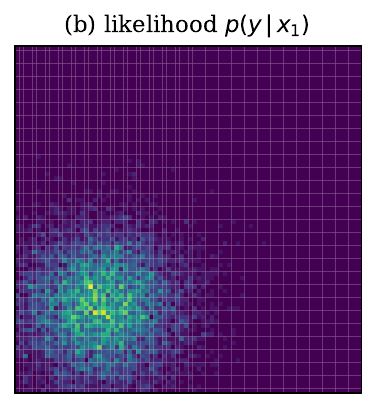}
        \caption{Likelihood.}\label{fig:an_like}
    \end{subfigure}\hfill
    \begin{subfigure}{0.24\linewidth}\centering
        \includegraphics[width=\linewidth]{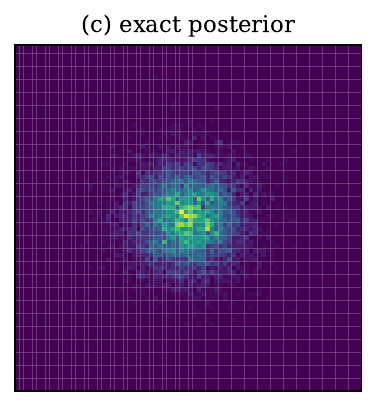}
        \caption{Exact posterior.}\label{fig:an_true}
    \end{subfigure}\hfill
    \begin{subfigure}{0.24\linewidth}\centering
        \includegraphics[width=\linewidth]{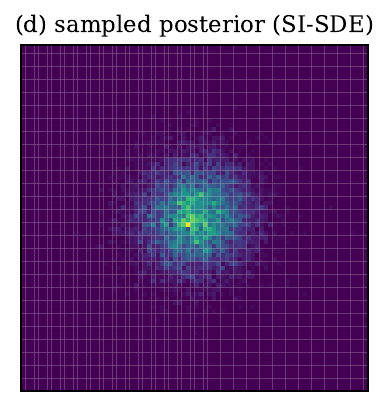}
        \caption{Sampled posterior (SI-SDE).}\label{fig:an_sampled}
    \end{subfigure}
    \caption{Analytical linear--Gaussian case.}
    \label{fig:analytical_panels}
\end{figure*}

\begin{table}[t]
    \centering
    \small
    \setlength{\tabcolsep}{4.5pt}
    \caption{Analytical case: Gaussian-fit KL divergence and sliced-$W_1$ to the exact posterior at $M=50$ and $M=500$ sampler steps (mean over five seeds, 100,000-sample ensembles), and wall-clock seconds per assimilation at $M=500$. The EnKF and particle filter perform a single analysis step (no $M$); both operate in their ideal linear--Gaussian regime here. A dash marks a collapsed ensemble with no finite KL.}
    \label{tab:analytical_accuracy}
    \begin{tabular}{l|cc|cc|c}
        \toprule
        & \multicolumn{2}{c|}{\textbf{Gaussian KL}} & \multicolumn{2}{c|}{\textbf{Sliced-$W_1$}} & \textbf{Cost} \\
        & $M{=}50$ & $M{=}500$ & $M{=}50$ & $M{=}500$ & s \\
        \midrule
        Ours (SI-SDE)                & 0.0004 & $6.8{\times}10^{-5}$ & 0.011 & 0.005 & 4.9 \\
        Ours (DM-SDE)                & 0.0012 & $8.2{\times}10^{-5}$ & 0.014 & 0.005 & 8.1 \\
        Ours (FM-ODE)                & 0.0018 & $8.0{\times}10^{-5}$ & 0.017 & 0.004 & 6.1 \\
        \midrule
        Ours (SI-SDE), Jac.-free     & 0.116  & 0.134  & 0.155 & 0.167 & 3.5 \\
        Ours (DM-SDE), Jac.-free     & 0.171  & 0.161  & 0.182 & 0.183 & 3.8 \\
        Ours (FM-ODE), Jac.-free     & 0.241  & 0.224  & 0.208 & 0.209 & 1.6 \\
        \midrule
        FlowDAS                      & 0.491  & 3.78   & 0.481 & 0.438 & 67.7 \\
        FlowDAS + SURGE              & $4.7{\times}10^{-5}$ & 87.8 & 0.005 & 0.649 & 69.1 \\
        SDA                          & 0.270  & 0.263  & 0.322 & 0.318 & 6.1 \\
        SDA + SURGE                  & 0.173  & 2.36   & 0.150 & 0.339 & 8.5 \\
        D-Flow SGLD                  & 0.096  & 56.1   & 0.190 & 0.628 & 9.9 \\
        Guided FM (FIG)              & 0.470  & --     & 0.354 & 0.654 & 1.9 \\
        \midrule
        EnKF                         & \multicolumn{2}{c|}{$7.3{\times}10^{-5}$} & \multicolumn{2}{c|}{0.005} & 0.01 \\
        Particle filter              & \multicolumn{2}{c|}{$6.4{\times}10^{-5}$} & \multicolumn{2}{c|}{0.004} & 0.01 \\
        \bottomrule
    \end{tabular}
\end{table}

\section{Navier--Stokes test case: setup and additional results}
\label{appendix:ns_case}

This appendix outlines the setup in greater detail and presents additional results: the data generation and forcing, qualitative reconstructions for each observation scenario, convergence in the number of sampler steps $M$, per-assimilation-step error curves, the enstrophy spectra of the reconstructions, and the ablations. Throughout: five held-out trajectories, $E=64$ ensemble members, 15 assimilation steps, observation noise $\obsnoise\sim\mathcal{N}(0,0.05^2\bI)$.

\subsection{Forcing and data generation}
\label{subsec:ns_data}

The stochastic forcing in Eq.~\eqref{eq:ns_equation} is temporally white and supplied to four Fourier modes,
\begin{equation}\label{eq:ns_forcing}
    \rd\xi(t,\boldsymbol{x})
    = \sum_{\boldsymbol{k}\in\mathcal{K}}
    \left[\cos(\boldsymbol{k}\cdot\boldsymbol{x})\,\rd W_t^{\boldsymbol{k},c}
    + \sin(\boldsymbol{k}\cdot\boldsymbol{x})\,\rd W_t^{\boldsymbol{k},s}\right],
    \qquad
    \mathcal{K}=\{(6,0),(7,0),(5,5),(8,8)\},
\end{equation}
with independent standard Wiener processes $W^{\boldsymbol{k},c}$ and $W^{\boldsymbol{k},s}$ and forcing amplitude $\varepsilon=1$. 

The trajectories are simulated with JAX-CFD~\citep{jax_cfd_2021} with a pseudo-spectral (Fourier) discretisation of Eq.~\eqref{eq:ns_equation} on a $256^2$ grid over the torus $\mathcal{T}^2=[0,2\pi]^2$, with viscosity $\nu=10^{-3}$ and linear drag $\alpha=0.1$. Time integration uses forward Euler--Maruyama on the vorticity Fourier coefficients with a fine step of $\rd t=10^{-4}$. Initial conditions are random divergence-free velocity fields with peak wavenumber $k=4$ and maximum velocity $7$.

Each trajectory is integrated for $500000$ time steps, corresponding to a time horizon of $T=50$, and stored every $0.5$ time units, i.e at every 5000 time step. This yields $100$ coarse snapshots per trajectory, of which the first $50$ are discarded so that the retained states are statistically spun up. The stored fields are finally down-sampled from $256^2$ to the $128^2$ resolution used for training and assimilation, so that the observation operators of Section~\ref{sec:results} act on a $128^2$ state.

\subsection{Ablations}
\label{subsec:results_ablations}
On the Navier--Stokes case, we study the computational approximations introduced in Section~\ref{subsec:approximations}. For the ensemble-shared covariance, we vary the Jacobian refresh rate $k$ to quantify the trade-off between accuracy and computational cost. Table~\ref{tab:ablation} collects the results.

\begin{table*}[ht]
    \centering
    \scriptsize
    \setlength{\tabcolsep}{3.5pt}
    \caption{Navier--Stokes ablation of the covariance approximation of Section~\ref{subsec:approximations}: the ensemble-shared inflated covariance at Jacobian refresh cadences $k\in\{1,5,10\}$ (the shared Jacobian is recomputed every $k$ of the $M$ sampler steps) against the Jacobian-free covariance. Same metrics and scenarios as Table~\ref{tab:ns_accuracy}, with per-step cost (lower is better). $E=64$ and $M=50$ sampler steps. \textbf{Bold} marks the best value in each column, together with every value equal to it to 2 decimal places in unscaled units.}
    \label{tab:ablation}
    \vspace{-10pt}
    \resizebox{\ifdim\width>\textwidth\textwidth\else\width\fi}{!}{%
    \begin{tabular}{l|cccc|cccc|cccc|cccc}
        \toprule
        & \multicolumn{4}{c|}{\textbf{Vorticity RMSE ($\times 10^{-1}$)}}
        & \multicolumn{4}{c|}{\textbf{CRPS ($\times 10^{-1}$)}}
        & \multicolumn{4}{c|}{\textbf{Spread--skill ($\times 10^{-1}$)}}
        & \multicolumn{4}{c}{\textbf{Cost (s/step)}} \\
        & $16^2$ & $32^2$ & $5\%$ & $\tfrac{1}{64}$ & $16^2$ & $32^2$ & $5\%$ & $\tfrac{1}{64}$ & $16^2$ & $32^2$ & $5\%$ & $\tfrac{1}{64}$ & $16^2$ & $32^2$ & $5\%$ & $\tfrac{1}{64}$ \\
        \midrule
        \multicolumn{17}{l}{\textit{Ensemble-shared inflated covariance -- Jacobian refreshed every step ($k=1$)}} \\
        Ours (SI-SDE) & 2.087 & 0.701 & \textbf{1.330} & \textbf{2.885} & 1.105 & 0.367 & \textbf{0.672} & \textbf{1.532} & 2.742 & 2.509 & 2.577 & 2.503 & 64.9 & 215.2 & 174.5 & 65.5 \\
        Ours (DM-SDE) & 2.248 & \textbf{0.640} & 1.470 & 4.548 & 1.178 & \textbf{0.334} & \textbf{0.740} & 2.480 & 1.572 & \textbf{0.637} & \textbf{1.079} & 1.929 & 134.2 & 442.4 & 360.3 & 135.2 \\
        Ours (FM-ODE) & \textbf{1.895} & \textbf{0.607} & 1.612 & 4.870 & \textbf{1.005} & \textbf{0.318} & 0.821 & 2.675 & \textbf{1.448} & 1.110 & 1.385 & 1.996 & 134.8 & 440.4 & 360.4 & 135.3 \\
        \midrule
        \multicolumn{17}{l}{\textit{Ensemble-shared inflated covariance -- Jacobian refreshed every $k=5$ steps}} \\
        Ours (SI-SDE) & 2.076 & 0.689 & 1.727 & 3.925 & 1.104 & 0.362 & 0.886 & 2.157 & 2.827 & 2.701 & 2.676 & 2.861 & 25.8 & 56.2 & 48.3 & 25.8 \\
        Ours (DM-SDE) & 2.166 & \textbf{0.642} & 3.125 & 6.621 & 1.157 & \textbf{0.335} & 1.656 & 3.723 & 1.490 & 0.796 & 1.639 & 2.111 & 53.5 & 115.9 & 99.7 & 53.5 \\
        Ours (FM-ODE) & 2.085 & \textbf{0.623} & 5.506 & 6.839 & 1.121 & \textbf{0.327} & 2.784 & 3.849 & 1.651 & 1.632 & 1.986 & 2.085 & 53.6 & 116.0 & 99.8 & 53.7 \\
        \midrule
        \multicolumn{17}{l}{\textit{Ensemble-shared inflated covariance -- Jacobian refreshed every $k=10$ steps}} \\
        Ours (SI-SDE) & 2.142 & 0.684 & 2.247 & 5.066 & 1.152 & 0.360 & 1.182 & 2.873 & 2.825 & 2.822 & 2.943 & 3.269 & 20.7 & 38.8 & 34.2 & 20.9 \\
        Ours (DM-SDE) & 2.344 & \textbf{0.649} & 4.389 & 7.339 & 1.261 & \textbf{0.339} & 2.387 & 4.138 & 1.557 & 0.921 & 1.909 & 1.988 & 43.6 & 82.1 & 72.2 & 44.3 \\
        Ours (FM-ODE) & 2.279 & \textbf{0.641} & 5.263 & 7.517 & 1.234 & \textbf{0.336} & 2.868 & 4.236 & 1.656 & 1.843 & 2.061 & 1.933 & 44.5 & 82.4 & 72.4 & 44.6 \\
        \midrule
        \multicolumn{17}{l}{\textit{Jacobian-free covariance}} \\
        Ours (SI-SDE) & 2.905 & 0.693 & 7.069 & 8.520 & 1.618 & 0.373 & 4.079 & 4.910 & 2.694 & 3.368 & 3.256 & 2.811 & \textbf{3.4} & \textbf{3.6} & \textbf{3.6} & \textbf{3.5} \\
        Ours (DM-SDE) & 2.845 & 0.679 & 6.805 & 8.370 & 1.553 & 0.354 & 3.804 & 4.705 & 1.623 & 1.209 & 2.172 & \textbf{1.539} & 6.5 & 6.6 & 6.6 & 6.5 \\
        Ours (FM-ODE) & 2.817 & 0.694 & 7.076 & 8.471 & 1.545 & 0.365 & 3.948 & 4.756 & 1.670 & 2.022 & 2.114 & \textbf{1.478} & 6.5 & 6.6 & 6.6 & 6.5 \\
        \bottomrule
    \end{tabular}}
  \end{table*}

\newcommand{\nspanel}[1]{%
  {\includegraphics[width=\linewidth]{figures/navier_stokes/singles/#1.pdf}}%
}
\newcommand{\nsqualpanel}[1]{%
  {\includegraphics[height=0.116\textheight]{figures/navier_stokes/singles/#1.pdf}}%
}
\newcommand{\nscbar}[1]{%
  \includegraphics[height=0.116\textheight]{figures/navier_stokes/singles/#1.pdf}%
}
\newcommand{\nsfourrow}[8]{%
  \raisebox{0.053\textheight}{\rotatebox[origin=c]{90}{\textbf{\scriptsize #3}}} &
  \nsqualpanel{ns_states_#1__#2__#4} &
  \nsqualpanel{ns_states_#1__#2__#5} &
  \nsqualpanel{ns_states_#1__#2__#6} &
  \nsqualpanel{ns_states_#1__#2__#7} &
  \nscbar{ns_states_#1__cbar_#8} \\
}
\newcommand{\nssettingfigure}[3]{%
\begin{figure*}[p]
    \centering
    \captionsetup{font=small}
    \setlength{\tabcolsep}{1.8pt}
    \renewcommand{\arraystretch}{0.78}
    \begin{tabular}{@{}rccccc@{}}
      & \multicolumn{2}{c}{\nsqualpanel{ns_states_#1__truth}} &
        \multicolumn{2}{c}{\nsqualpanel{ns_obs_#1}} & \\
      & \multicolumn{2}{c}{\textbf{\scriptsize True state}} &
        \multicolumn{2}{c}{\textbf{\scriptsize Observations}} & \\[3pt]
      \textbf{\scriptsize Ours} &
        \textbf{\scriptsize SI-SDE (shared)} &
        \textbf{\scriptsize FM-ODE (shared)} &
        \textbf{\scriptsize SI-SDE (JF)} &
        \textbf{\scriptsize FM-ODE (JF)} & \\
      \nsfourrow{#1}{mean}{State}
        {ours_si_sde_shared}{ours_fm_ode_shared}
        {ours_si_sde_jac_free}{ours_fm_ode_jac_free}{field}
      \nsfourrow{#1}{std}{Standard deviation}
        {ours_si_sde_shared}{ours_fm_ode_shared}
        {ours_si_sde_jac_free}{ours_fm_ode_jac_free}{std}
      \nsfourrow{#1}{abserr}{Absolute error}
        {ours_si_sde_shared}{ours_fm_ode_shared}
        {ours_si_sde_jac_free}{ours_fm_ode_jac_free}{abserr}
      \noalign{\vskip 4pt}
      \textbf{\scriptsize Baselines} &
        \textbf{\scriptsize FlowDAS+SURGE} &
        \textbf{\scriptsize SDA+SURGE} &
        \textbf{\scriptsize D-Flow SGLD} &
        \textbf{\scriptsize Guided FM (FIG)} & \\
      \nsfourrow{#1}{mean}{State}
        {flowdas_surge}{sda_surge}{d_flow_sgld}{guided_fm_fig}{field}
      \nsfourrow{#1}{std}{Standard deviation}
        {flowdas_surge}{sda_surge}{d_flow_sgld}{guided_fm_fig}{std}
      \nsfourrow{#1}{abserr}{Absolute error}
        {flowdas_surge}{sda_surge}{d_flow_sgld}{guided_fm_fig}{abserr}
    \end{tabular}
    \caption{Navier--Stokes, #2: true vorticity and observations (top), followed by the posterior-mean state, ensemble standard deviation, and absolute error for our methods and the retained baselines at the final assimilation step. ``Shared'' denotes the ensemble-shared covariance and ``JF'' the Jacobian-free covariance.}
    \label{#3}
\end{figure*}%
}

\nssettingfigure{32_2_128_2}{super-resolution $32^2\!\to\!128^2$}{fig:ns_fields_32}
\nssettingfigure{16_2_128_2}{super-resolution $16^2\!\to\!128^2$}{fig:ns_fields_16}
\nssettingfigure{sparse_5}{sparse sensors ($5\%$)}{fig:ns_fields_sparse5}
\nssettingfigure{sparse_1_5625}{sparse sensors ($1.5625\%$)}{fig:ns_fields_sparse1p5}

\newcommand{\nscurvepanel}[3]{%
    \begin{subfigure}[t]{0.42\linewidth}\centering
        \includegraphics[height=0.155\textheight]{figures/navier_stokes/singles/ns_#1_#2.pdf}
        \vspace{2pt}
        \caption{#3}
    \end{subfigure}%
}
\newcommand{\nsmetricfigure}[3]{%
\begin{figure*}[p]
    \centering
    \captionsetup{font=small}
    \captionsetup[subfigure]{skip=4pt}
    \textbf{\small Versus sampler steps $M$}\par\vspace{2pt}
    \nscurvepanel{#1_vs_M}{32_2_128_2}{Super-resolution $32^2\!\to\!128^2$.}
    \hspace{0.04\linewidth}
    \nscurvepanel{#1_vs_M}{16_2_128_2}{Super-resolution $16^2\!\to\!128^2$.}
    \par\vspace{4pt}
    \nscurvepanel{#1_vs_M}{sparse_5}{Sparse sensors ($5\%$).}
    \hspace{0.04\linewidth}
    \nscurvepanel{#1_vs_M}{sparse_1_5625}{Sparse sensors ($1.5625\%$).}
    \par\vspace{8pt}
    \textbf{\small Versus assimilation step}\par\vspace{2pt}
    \nscurvepanel{#1_vs_step}{32_2_128_2}{Super-resolution $32^2\!\to\!128^2$.}
    \hspace{0.04\linewidth}
    \nscurvepanel{#1_vs_step}{16_2_128_2}{Super-resolution $16^2\!\to\!128^2$.}
    \par\vspace{4pt}
    \nscurvepanel{#1_vs_step}{sparse_5}{Sparse sensors ($5\%$).}
    \hspace{0.04\linewidth}
    \nscurvepanel{#1_vs_step}{sparse_1_5625}{Sparse sensors ($1.5625\%$).}
    \par\vspace{4pt}
    \includegraphics[width=0.8\linewidth]{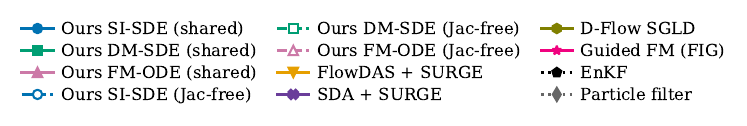}
    \vspace{-5pt}
    \caption{Navier--Stokes: #2 as a function of the number of sampler steps $M$ (top two rows) and assimilation step at the largest $M$ (bottom two rows), for all observation scenarios.}
    \label{#3}
\end{figure*}%
}

\nsmetricfigure{rmse}{ensemble-mean vorticity RMSE}{fig:ns_rmse_curves}
\nsmetricfigure{crps}{CRPS}{fig:ns_crps_curves}
\nsmetricfigure{spread_skill}{Spread-skill}{fig:ns_spread_skill_curves}

\subsection{Enstrophy spectra}
\label{subsec:ns_spectra}

Pointwise metrics do not reveal which scales are misrepresented. Figure~\ref{fig:ns_enstrophy_spectra} therefore compares the radially averaged enstrophy spectrum of each posterior mean with the truth:
\begin{equation}\label{eq:ns_enstrophy_spectrum}
Z(k)=\left\langle \tfrac12|\hat{\vorticity}(\boldsymbol{k})|^2 \right\rangle_{|\boldsymbol{k}|\in\mathrm{shell}(k)},
\end{equation}
using 60 radial bins and averaging over the 15 assimilated steps. This diagnostic is computed directly from vorticity.

\newcommand{\nsspectrumpanel}[3]{%
    \begin{subfigure}[t]{0.42\linewidth}\centering
        {\includegraphics[height=0.155\textheight]{figures/navier_stokes/singles/ns_enstrophy_spectrum#1_#2.pdf}}%
        \vspace{2pt}
        \caption{#3}
    \end{subfigure}%
}
\newcommand{\nsspectrumrows}[1]{%
    \nsspectrumpanel{#1}{32_2_128_2}{Super-resolution $32^2\!\to\!128^2$.}
    \hspace{0.04\linewidth}
    \nsspectrumpanel{#1}{16_2_128_2}{Super-resolution $16^2\!\to\!128^2$.}
    \par\vspace{4pt}
    \nsspectrumpanel{#1}{sparse_5}{Sparse sensors ($5\%$).}
    \hspace{0.04\linewidth}
    \nsspectrumpanel{#1}{sparse_1_5625}{Sparse sensors ($1.5625\%$).}%
}
\begin{figure*}[p]
    \centering
    \captionsetup{font=small}
    \captionsetup[subfigure]{skip=4pt}
    \textbf{\small Full spectrum}\par\vspace{2pt}
    \nsspectrumrows{}
    \par\vspace{8pt}
    \textbf{\small High-wavenumber detail ($k\ge30$)}\par\vspace{2pt}
    \nsspectrumrows{_zoom}
    \par\vspace{4pt}
    \includegraphics[width=0.8\linewidth]{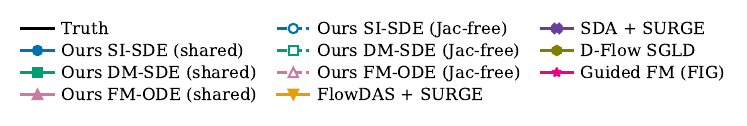}
    \vspace{-5pt}
    \caption{Navier--Stokes: radially-averaged enstrophy spectra $Z(k)$ of the posterior mean for every method (colours) against the true spectrum (black), for all four observation scenarios. Top two rows: the full wavenumber range. Bottom two rows: the same panels zoomed to $k\ge30$, where the methods separate -- a curve below the truth has an over-smoothed posterior mean, one above it carries spurious small-scale enstrophy. Spectra are averaged over the 15 assimilated steps of one held-out trajectory at the largest sampler-step count $M$. A method with no saved run at this $M$, or whose posterior diverged there, carries no curve.}
    \label{fig:ns_enstrophy_spectra}
\end{figure*}

\section{Urban airflow test case: setup and additional results}
\label{appendix:urban_case}

This appendix provides additional information and results for the urban-airflow case, including qualitative reconstructions and error trajectories over 145 forecast steps. The reported run uses one held-out trajectory, $E=64$ ensemble members, $M=50$ sampler steps, and sensor densities of $1.5625\%$ and $0.78125\%$ of fluid cells. Starting from a different held-out flow state, observations are assimilated every third forecast step, for 49 posterior updates. Metrics are evaluated at every step over fluid cells only, with buildings masked. Due to computational complexities, this case includes only solver-free generative methods; no conventional-filter or KL reference is reported.

\paragraph{The uDALES solver.}
The data are generated with uDALES~\citep{suter_udales_2022}, an open-source, building-resolving large-eddy simulation code for urban microclimate and air quality derived from DALES. It solves the filtered incompressible Navier--Stokes and scalar-transport equations on a staggered Cartesian grid, using the Vreman model for subgrid stresses and an immersed-boundary representation of buildings. Momentum, temperature, and building-surface energy are coupled within a periodic $128\times128\times128\,\mathrm{m}^3$ domain discretized at $1\,\mathrm{m}$ resolution.

\newcommand{\urbqualpanel}[1]{%
  {\includegraphics[height=0.116\textheight]{figures/urban/singles/#1.pdf}}%
}
\newcommand{\urbqualcbar}[1]{%
  {\includegraphics[height=0.116\textheight]{figures/urban/singles/#1.pdf}}%
}
\newcommand{\urbfourrow}[9]{%
  \raisebox{0.053\textheight}{\rotatebox[origin=c]{90}{\textbf{\scriptsize #4}}} &
  \urbqualpanel{urban_states_#1_#2__#3__#5} &
  \urbqualpanel{urban_states_#1_#2__#3__#6} &
  \urbqualpanel{urban_states_#1_#2__#3__#7} &
  \urbqualpanel{urban_states_#1_#2__#3__#8} &
  \urbqualcbar{urban_states_#1_#2__cbar_#9} \\
}
\newcommand{\urbsettingfigure}[5]{%
\begin{figure*}[p]
    \centering
    \captionsetup{font=small}
    \setlength{\tabcolsep}{1.8pt}
    \renewcommand{\arraystretch}{0.78}
    \begin{tabular}{@{}rccccc@{}}
      & \multicolumn{2}{c}{\urbqualpanel{urban_truth_#1}} &
        \multicolumn{2}{c}{\urbqualpanel{urban_obs_#2}} & \\
      & \multicolumn{2}{c}{\textbf{\scriptsize True state}} &
        \multicolumn{2}{c}{\textbf{\scriptsize Observations}} & \\[3pt]
      \textbf{\scriptsize Ours} &
        \textbf{\scriptsize SI-SDE (shared)} &
        \textbf{\scriptsize FM-ODE (shared)} &
        \textbf{\scriptsize SI-SDE (JF)} &
        \textbf{\scriptsize FM-ODE (JF)} & \\
      \urbfourrow{#1}{#2}{mean}{State}
        {ours_si_sde_shared}{ours_fm_ode_shared}
        {ours_si_sde_jac_free}{ours_fm_ode_jac_free}{field}
      \urbfourrow{#1}{#2}{std}{Standard deviation}
        {ours_si_sde_shared}{ours_fm_ode_shared}
        {ours_si_sde_jac_free}{ours_fm_ode_jac_free}{std}
      \urbfourrow{#1}{#2}{abserr}{Absolute error}
        {ours_si_sde_shared}{ours_fm_ode_shared}
        {ours_si_sde_jac_free}{ours_fm_ode_jac_free}{abserr}
      \noalign{\vskip 4pt}
      \textbf{\scriptsize Baselines} &
        \textbf{\scriptsize FlowDAS+SURGE} &
        \textbf{\scriptsize SDA+SURGE} &
        \textbf{\scriptsize D-Flow SGLD} &
        \textbf{\scriptsize Guided FM (FIG)} & \\
      \urbfourrow{#1}{#2}{mean}{State}
        {flowdas_surge}{sda_surge}{d_flow_sgld}{guided_fm_fig}{field}
      \urbfourrow{#1}{#2}{std}{Standard deviation}
        {flowdas_surge}{sda_surge}{d_flow_sgld}{guided_fm_fig}{std}
      \urbfourrow{#1}{#2}{abserr}{Absolute error}
        {flowdas_surge}{sda_surge}{d_flow_sgld}{guided_fm_fig}{abserr}
    \end{tabular}
    \caption{Urban airflow, #4, #3: true state and velocity observations (top), followed by the posterior-mean state, ensemble standard deviation, and absolute error for our methods and the retained baselines at the final forecast step. ``Shared'' denotes the ensemble-shared covariance and ``JF'' the Jacobian-free covariance; solid cells are masked.}
    \label{#5}
\end{figure*}%
}

\paragraph{Qualitative reconstructions.}
Figures~\ref{fig:urban_fields_vel_0p7}--\ref{fig:urban_fields_temp_1p5} show the true field, observations, reconstruction, uncertainty, and error for each retained method. Each figure corresponds to one variable and sensor density. As in the Navier--Stokes figures, we omit the DM-SDE variants and show only the SURGE versions of FlowDAS and SDA.

\urbsettingfigure{velocity}{sparse_0_78125}{sparse sensors ($0.78125\%$)}{velocity magnitude}{fig:urban_fields_vel_0p7}
\urbsettingfigure{temperature}{sparse_0_78125}{sparse sensors ($0.78125\%$)}{potential temperature $\theta_\ell$}{fig:urban_fields_temp_0p7}
\urbsettingfigure{velocity}{sparse_1_5625}{sparse sensors ($1.5625\%$)}{velocity magnitude}{fig:urban_fields_vel_1p5}
\urbsettingfigure{temperature}{sparse_1_5625}{sparse sensors ($1.5625\%$)}{potential temperature $\theta_\ell$}{fig:urban_fields_temp_1p5}

\newcommand{\urbcurve}[1]{%
  {\includegraphics[width=\linewidth,height=0.155\textheight,keepaspectratio]{figures/urban/singles/#1.pdf}}%
}
\newcommand{\urblegend}{%
  {\includegraphics[width=0.8\linewidth]{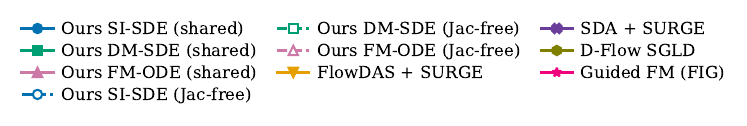}}%
}

\begin{figure*}[ht]
    \centering
    \begin{subfigure}[t]{0.4\linewidth}\centering
        \urbcurve{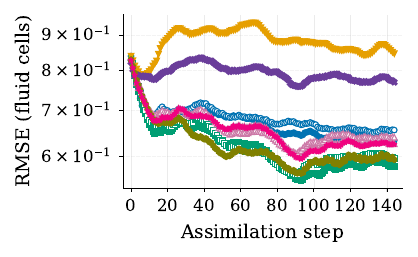}
        \vspace{-7pt}
        \caption{Velocity, sparse $0.78125\%$.}
    \end{subfigure}\hspace{0.04\linewidth}
    \begin{subfigure}[t]{0.4\linewidth}\centering
        \urbcurve{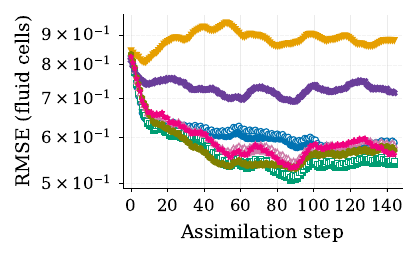}
        \vspace{-7pt}
        \caption{Velocity, sparse $1.5625\%$.}
    \end{subfigure}\hspace{0.04\linewidth}
    \begin{subfigure}[t]{0.4\linewidth}\centering
        \urbcurve{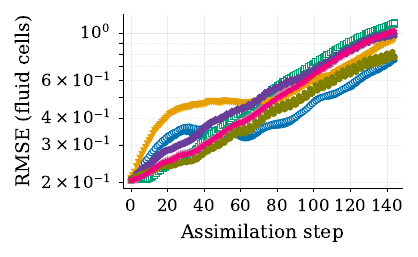}
        \vspace{-7pt}
        \caption{Temperature $\theta_\ell$, sparse $0.78125\%$.}
    \end{subfigure}\hspace{0.04\linewidth}
    \begin{subfigure}[t]{0.4\linewidth}\centering
        \urbcurve{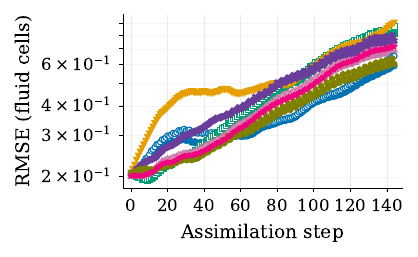}
        \vspace{-7pt}
        \caption{Temperature $\theta_\ell$, sparse $1.5625\%$.}
    \end{subfigure}
    \urblegend
    \vspace{-10pt}
    \caption{Urban airflow: RMSE (fluid cells) vs.\ forecast step for velocity and temperature at both sensor densities; observations are assimilated every third step.}
    \label{fig:urban_rmse_vs_step}
\end{figure*}

\begin{figure*}[ht]
    \centering
    \begin{subfigure}[t]{0.4\linewidth}\centering
        \urbcurve{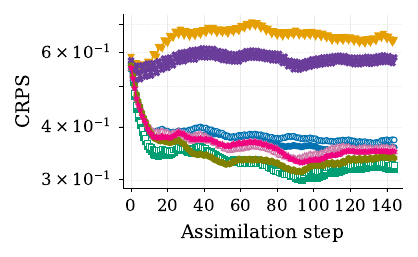}
        \vspace{-7pt}
        \caption{Velocity, sparse $0.78125\%$.}
    \end{subfigure}\hspace{0.04\linewidth}
    \begin{subfigure}[t]{0.4\linewidth}\centering
        \urbcurve{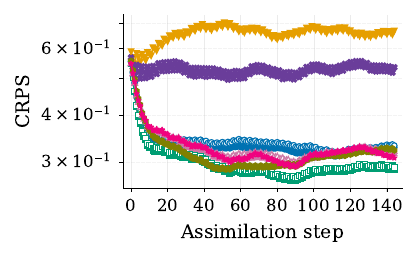}
        \vspace{-7pt}
        \caption{Velocity, sparse $1.5625\%$.}
    \end{subfigure}\hspace{0.04\linewidth}
    \begin{subfigure}[t]{0.4\linewidth}\centering
        \urbcurve{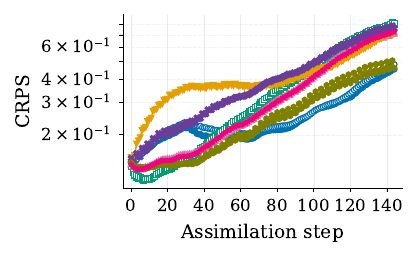}
        \vspace{-7pt}
        \caption{Temperature $\theta_\ell$, sparse $0.78125\%$.}
    \end{subfigure}\hspace{0.04\linewidth}
    \begin{subfigure}[t]{0.4\linewidth}\centering
        \urbcurve{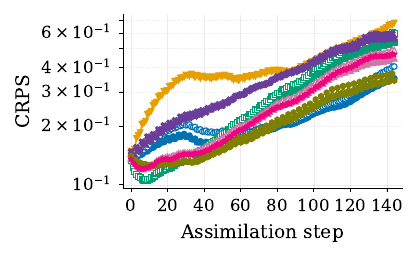}
        \vspace{-7pt}
        \caption{Temperature $\theta_\ell$, sparse $1.5625\%$.}
    \end{subfigure}
    \urblegend
    \vspace{-10pt}
    \caption{Urban airflow: CRPS vs.\ forecast step for velocity and temperature at both sensor densities; observations are assimilated every third step.}
    \label{fig:urban_crps_vs_step}
\end{figure*}

\begin{figure*}[ht]
    \centering
    \begin{subfigure}[t]{0.4\linewidth}\centering
        \urbcurve{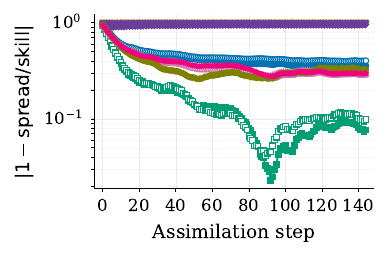}
        \vspace{-7pt}
        \caption{Velocity, sparse $0.78125\%$.}
    \end{subfigure}\hspace{0.04\linewidth}
    \begin{subfigure}[t]{0.4\linewidth}\centering
        \urbcurve{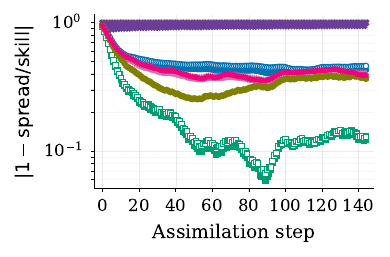}
        \vspace{-7pt}
        \caption{Velocity, sparse $1.5625\%$.}
    \end{subfigure}\hspace{0.04\linewidth}
    \begin{subfigure}[t]{0.4\linewidth}\centering
        \urbcurve{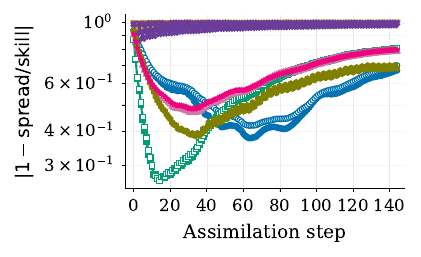}
        \vspace{-7pt}
        \caption{Temperature $\theta_\ell$, sparse $0.78125\%$.}
    \end{subfigure}\hspace{0.04\linewidth}
    \begin{subfigure}[t]{0.4\linewidth}\centering
        \urbcurve{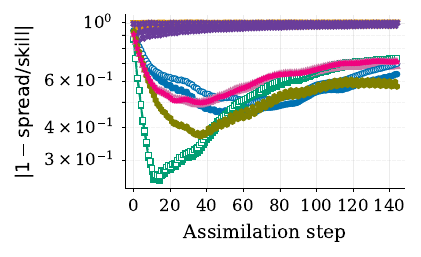}
        \vspace{-7pt}
        \caption{Temperature $\theta_\ell$, sparse $1.5625\%$.}
    \end{subfigure}
    \urblegend
    \vspace{-10pt}
    \caption{Urban airflow: spread--skill deviation vs.\ forecast step for velocity and temperature at both sensor densities; observations are assimilated every third step.}
    \label{fig:urban_ss_vs_step}
\end{figure*}

\end{document}